\documentclass[letterpaper]{article} 
\usepackage{aaai24}
\usepackage{times}  
\usepackage{helvet}  
\usepackage{courier}  
\usepackage[hyphens]{url}  
\usepackage{graphicx} 
\usepackage{xcolor}
\usepackage{natbib}  
\usepackage{caption} 
\usepackage{mathtools}
\usepackage{algorithm}
\usepackage{xspace}
\usepackage{amsmath,amssymb,amsfonts,amsthm}
\usepackage{bbold}
\usepackage{algpseudocode}
\usepackage[utf8]{inputenc}
\usepackage{subcaption}
\usepackage{enumitem}
\usepackage{colortbl}
\allowdisplaybreaks

\usepackage{newfloat}
\usepackage{listings}
\DeclareCaptionStyle{ruled}{labelfont=normalfont,labelsep=colon,strut=off} 
\floatstyle{ruled}
\newfloat{listing}{tb}{lst}{}
\floatname{listing}{Listing}
\DeclareRobustCommand{\wrt}{w.r.t.\@\xspace}

\newtheorem{theorem}{Theorem}
\newtheorem{remark}{Remark}

\newtheorem{lemma}{Lemma}

\newcommand{\abar}{\bar{A}}
\newcommand{\E}{\mathbb{E}}
\newcommand{\xtpast}{t-1:t-L}
\newcommand{\ytpast}{t-1:t-M}

\title{Causal Feature Selection via Transfer Entropy}
\author {
    Paolo Bonetti\textsuperscript{\rm 1},
    Alberto Maria Metelli\textsuperscript{\rm 1},
    Marcello Restelli\textsuperscript{\rm 1}
}
\affiliations {
    \textsuperscript{\rm 1}Politecnico di Milano\\
    paolo.bonetti@polimi.it, albertomaria.metelli@polimi.it, marcello.restelli@polimi.it
}

\begin{document}

\algnewcommand\algorithmicforeach{\textbf{for each}}
\algdef{S}[FOR]{ForEach}[1]{\algorithmicforeach\ #1\ \algorithmicdo}
\renewcommand{\algorithmicrequire}{\textbf{Input:}}
\renewcommand{\algorithmicensure}{\textbf{Output:}}

\maketitle

\begin{abstract}
Machine learning algorithms are designed to capture complex relationships between features. In this context, the high dimensionality of data often results in poor model performance, with the risk of overfitting. Feature selection, the process of selecting a subset of relevant and non-redundant features, is, therefore, an essential step to mitigate these issues. However, classical feature selection approaches do not inspect the causal relationship between selected features and target, which can lead to misleading results in real-world applications. Causal discovery, instead, aims to identify causal relationships between features with observational data. In this paper, we propose a novel methodology at the intersection between feature selection and causal discovery, focusing on time series. We introduce a new causal feature selection approach that relies on the forward and backward feature selection procedures and leverages transfer entropy to estimate the causal flow of information from the features to the target in time series. Our approach enables the selection of features not only in terms of mere model performance but also captures the causal information flow. 
In this context, we provide theoretical guarantees on the regression and classification errors for both the exact and the finite-sample cases. Finally, we present numerical validations on synthetic and real-world regression problems, showing results competitive w.r.t. the considered baselines. 
\end{abstract}

\section{Introduction}
Machine learning~\citep[ML,][]{bishop2006} techniques are designed for predictive modelling, data analysis, and decision-making and are applied in numerous fields such as medicine, finance, and social sciences. One of the most significant challenges in ML is dealing with \emph{observational} data, which is collected without controlling for the variables of interest or the confounding variables. This type of data can lead to biased results, hindering the generalizability and interpretability of the ML models. To reduce the number of dimensions and avoid memory and modelling issues like the curse of dimensionality, ML models are trained over a set of variables identified by \emph{feature selection}~\citep{li2017feature}, which usually chooses the features with the largest association with the target, without inspecting the \emph{causal} relationships ~\citep{peters2017elements}. 

\emph{Causal discovery} is the process of identifying causal relationships among variables from observational data without external interventions~\citep{spirtes2000,eberhardt2017}. Its primary purpose is to undercover the direction of causality and estimate the magnitude and uncertainty of causal effects. This approach not only helps in developing more accurate and robust models but also provides a deeper understanding of the underlying mechanisms governing the data~\citep{pearl2009}. 

The importance of observational causal discovery for time series data in ML lies on the fact that a model trained on causal features enables researchers to predict the outcomes of an intervention or a policy change in a robust and interpretable way~\citep{scholkopf2022causality}. This is particularly useful in situations where it is not feasible or ethical to intervene, such as in healthcare, social sciences, or economics~\citep{imbens2015causal}. Furthermore, considering a physical system such as climate science, the causal knowledge of direct and indirect effects among variables can help to understand the physical system, filtering the effect of spurious correlations. Causal discovery can also guide the choice of the features by performing a \emph{causal} feature selection that identifies those with the most relevant flow of information to the target. However, causal discovery for time series is challenging on complex systems due to unobserved confounders, high dimensionality, autocorrelated variables, and time lags~\citep{runge2019inferring}. Moreover, causal discovery algorithms usually focus on theoretical guarantees on the (asymptotic) convergence to the real causal graph, or to an equivalence class of graphs, without inspecting the effect of the selection of the candidate causes of a variable on its prediction. 

\paragraph{Contributions}
In this paper, we propose a novel methodology at the intersection between feature selection and causal discovery. We start in Section \ref{sec:relWorks} by reviewing the main approaches and assumptions of causal discovery for time series. Then, in Section \ref{sec:preliminaries}, we introduce the notation and formulation of the problem. In Section \ref{sec:TEFS}, we introduce the novel \emph{causal} feature selection methodology that relies on the \emph{forward} and \emph{backward} feature selection procedures and leverages the \emph{transfer entropy} (TE)~\cite{schreiber2000measuring} as building block to guide the selection of the features. TE is a causal quantity specifically designed to estimate the flow of information between features and the target in time series. 
Unlike traditional feature selection and similar to Granger causality, our causal approach exploits the property of TE to filter out the autoregressive component of the target in order to measure the asymmetric flow of information from a feature to it without the confounding association due to the target itself. Furthermore, conditioning on the other features also makes it possible to filter out the information due to the presence of a causal variable that a feature and the target have in common.
In this setting, we provide theoretical guarantees about the regression and classification error due to the reduction of features. In addition, in Section \ref{sec:TEinv}, we analyse the finite-sample scenario, exploiting a concentration bound of a specific kernel-based TE estimator available in the literature~\citep{singh2014exponential}.  Finally, in Section \ref{sec:experiments}, we provide numerical simulations in regression settings, showing the capability of the proposed approaches to identify causal features in synthetic experiments and to achieve competitive performance on real-world datasets. 

\section{Related Works}\label{sec:relWorks}
In this section, we review state-of-the-art causal discovery approaches, focusing on their application to time series.{\color{blue} \footnote{Regarding feature selection methods, we refer the interested reader to the survey~\citet{li2017feature} for an overview.}}

\paragraph{Classical Causal Discovery}Classical causal discovery relies on four main assumptions~\citep{pearl2000}.
\emph{Causal Markov} and \emph{causal faithfulness} ensure a correspondence between the (conditional) probability independence between variables and the absence of a causal relationship in the causal graph. Then, \emph{acyclicity} avoids the presence of cyclical causal relationships, and \emph{causal sufficiency} guarantees the absence of latent confounders. The benchmark non-parametric algorithm of causal discovery that assumes all these four assumptions to hold is the PC algorithm~\citep{spirtes2000}. Non-parametric approaches that relax parts of the assumptions or parametric approaches that assume a linear or non-linear structural causal model are revised in~\citep{eberhardt2017,glymour2019review}.

\paragraph{Causal Discovery for Time Series}
In ~\citep{runge2019inferring}, the difficulties of extending the causal discovery framework to high-dimensional time series are extensively analysed. The classic extension of causal discovery to time series is the notion of \emph{Granger causality}~\cite{granger1969investigating}, which evaluates the impact of the past value of a feature on the prediction of the target as its causal strength~\citep{ding2006granger}. This approach may suffer from the curse of dimensionality and autocorrelation effects, as highlighted in~\citep{runge2019detecting}, where the PCMCI algorithm has been proposed to deal with these effects. Variations of Granger causality-based methods and of the PCMCI algorithm that deals with contemporaneous causal effects and latent confounders can be found in~\citep{moneta2011causal,malinsky2018causal,runge2020discovering,gerhardus2020high}. Other existing approaches assume specific structures of the causal models or apply neural networks to extend the methods to highly non-linear contexts. A broader overview can be found in~\citet{moraffah2021causal}.

Similar to the approaches proposed in this paper, some recent works ~\citep{mastakouri2020causal,mastakouri2021necessary,dong2022geass} focus on identifying sets of causally relevant features rather than the complete causal graph, although they are not directly comparable because of different settings and assumptions. In particular, ~\citep{mastakouri2020causal,mastakouri2021necessary} introduce and apply the \emph{SyPI} algorithm, which is guaranteed to asymptotically identify (some) ancestors of a target variable, also in the presence of latent confounders, under several assumptions on the relationships between variables. \citep{dong2022geass} focus on the identification of a set of causal features, exploiting a neural network to optimize a specifically designed loss to asymptotically guarantee the identification of a set of features causally interconnected.

\paragraph{Causal Information Measures}
Several indexes from information theory are available to quantify the causal information that flows from one variable to another. 
\emph{Transfer entropy}~\citep[TE,][]{schreiber2000measuring} evaluates the information flow from one variable to another, computing the information that the past values of the first have on the second. Similarly, \emph{directed information}~\citep{massey1990causality} computes the cumulative sum of information shared between all the past values of the first variable and the actual value of the second one. These two quantities are strictly related~\citep{liu2012relationship}, and they can be applied to test for Granger causality between variables~\citep{amblard2012relation}.

\section{Preliminaries}\label{sec:preliminaries}

\paragraph{Time Series}
We consider a supervised learning problem with $D$ features and a scalar target. At the time $t$, we denote with $Y_t$ the target's current random value and with $X_i^t$ the current random value of the $i$-th feature. Features and target are multivariate time series $X_i=(X_i^{t},X_i^{t-1},\dots)$ for $i\in [D] \coloneqq \{1,\dots,D\}$ and $Y=(Y_{t},Y_{t-1},\dots)$ that form a nonlinear, discrete-time, stationary vector autoregressive process. We denote with $X_A^{\xtpast}$ the last $L$ values of the set of features having indices in $A\subseteq [D]$ and with $X_{A}^{\xtpast}$ the last $L$ of the set of features having indices in set $A$ (for simplicity, $X^{\xtpast} = X_{[D]}^{\xtpast}$ and $X_i^{\xtpast} = X_{\{i\}}^{\xtpast}$). We also denote $X_A\cup X_i := X_{A\cup \{i\}}$ and $X_A\setminus X_i := X_{A\setminus \{i\}}$. Finally, $Y_{\ytpast}$ is the set of the last $M$ values of the target. We assume that the absolute value of the target is bounded by a constant $B$ uniformly over time. We denote with $\mathbb{E}_Z[f(Z)]$ and $\mathbb{V}\mathrm{ar}_Z[f(Z)]$ the expected value and the variance of a function $f(\cdot)$ of the random variable $Z$ \wrt\ its distribution.

\paragraph{Causal Model}
Together with the four classical assumptions of causal discovery introduced in the previous section (causal sufficiency, acyclicity, causal Markov assumption, faithfulness), we assume that the actual target may only depend on its last $M$ values and the last $L$ values of the features. 
The (stationary) autoregressive processes are described by the following \emph{structural causal model} (SCM):
\begin{align*}
    &X_i^t = f_i(\text{pa}(X_i^t),\epsilon_i), \quad i\in [D], \qquad
    Y_t = f(\text{pa}(Y_t),\epsilon),
\end{align*}
where $f_i, f$ control the relationship between each variable and its causes, $\epsilon_i,\epsilon$ are independent unobserved noise terms and $\text{pa}(X_i^t)$ (resp. $\text{pa}(Y_t)$) is the set of unknown causes of feature $X_i$ (resp. the target $Y_t$), assumed to depend only on some past values of the considered features (resp. target). 

\paragraph{Information Theory}
\emph{Mutual information}~\citep[MI,][]{Cover2006} is an index considered in information theory to evaluate the amount of information that a feature $X$ shares with the target $Y$, defined as: 
\begin{align*}
    I(X;Y)&\coloneqq \mathbb{E}_{X,Y}\left[\log \frac{p(X,Y)}{p(X)p(Y)}\right],  
\end{align*}
where $p(\cdot,\cdot)$ is the joint p.d.f. of the random variables $(X,Y)$. \emph{Conditional mutual information} (CMI) is obtained conditioning w.r.t. a third random variable $Z$, $I(X;Y|Z)\coloneqq\E_Z[I(X|Z;Y|Z)]$.
\emph{Transfer entropy}~\citep[TE,][]{schreiber2000measuring} evaluates the information flow from one variable to another, computing the information that the past values of the first have on the second. This is the directional counterpart of mutual information, which is symmetric and, therefore, cannot identify the causal direction of the flow of information. 
Similarly to CMI, conditional transfer entropy~\citep{Shahsavari2020} extends the definition of TE to a conditional set that only considers the additional amount of information flow from $X$ to $Y$ given the flow from $Z$ to $Y$. We denote the TE between $X$ and $Y$ as $TE_{X\rightarrow Y}:=I(Y_t;X^{\xtpast}|Y_{\ytpast})$ and the conditional transfer entropy between $X$ and $Y$ given $Z$ as $TE_{X\rightarrow Y| Z}:=I(Y_t;X^{\xtpast}|Y_{\ytpast},Z^{\xtpast})$. $M$ and $L$ determine the maximum number of previous values of the features and target that impact the current value of the target.


\section{Transfer Entropy for Causal Feature Selection}\label{sec:TEFS}
In this section, we introduce two theoretical results that relate the effect of selecting features based on the TE to the regression and classification errors (Section~\ref{subsec:TEFStheo}). Then, we apply these findings to devise novel forward and backward \emph{causal} feature selection algorithms (Section~\ref{sec:TEFSalgo}). The results we present extend the findings of~\citep{Beraha2019} in two directions: (i) we move from an i.i.d. supervised learning setting to a time series setting, (ii) we provide a result that is expressed in terms of TE, which represents a natural expression of the loss of causal information for time series.


\subsection{Regression and Classification Error Bounds}\label{subsec:TEFStheo}


We first provide an upper bound to the Mean Squared Error (MSE) of the regression performed to estimate the value of the target $Y_t$ at time $t$, considering the last $M$ values of the target and the last $L$ values of the selected features.\footnote{Note that the bound holds for a generic subset of features, not necessarily the result of a feature selection.}  The proofs of the results are reported in Appendix \ref{app:TEFS}.

\begin{theorem}\label{thm:regrErr}
Let $ \Bar{A} \subseteq [D]$ be a set of indices and let $A = [D] \setminus \Bar{A}$ be its complementary. 
Then, considering a lag of $L,M \in \mathbb{N}_{\ge 0}$ for the features and the target respectively, the regression error suffered by considering only the set of features $X_{\Bar{A}}$ is bounded by:
\begin{equation}\label{eq:regrErr}
\begin{aligned}
    \inf_{g\in \mathcal{G}_{\abar}}\E_{X,Y}[(Y_t-g(X_{\abar}^{\xtpast},Y_{\ytpast}))^2] \\ \leq \sigma^2 + 2B^2\cdot TE_{X_{A}\rightarrow Y| X_{\abar}},
\end{aligned}
\end{equation}

where $\mathcal{G}_{\abar}$ is the set of all functions that map $X_{\abar}$ to $\mathbb{R}$, {\fontsize{8.5}{9}$\sigma^2 = \E_{X,Y}[(Y_t-\E[Y_t|X^{\xtpast},Y_{\ytpast})^2]$} is the irreducible error, and $B$ is the maximum absolute value of the target.
\end{theorem}

The result of the theorem complies with the intuition that, considering a subset of features, the expected MSE is bounded by the irreducible expected MSE (unavoidable even when considering the entire set of features), plus an index of the information flow from the discarded features to the target $TE_{X_{A}\rightarrow Y| X_{\abar}}$.
A similar result holds for the Bayes error in classification settings.
\begin{theorem}\label{thm:classErr}
    Let $A$ and $\abar$ be defined as in Theorem \ref{thm:regrErr}. Then, considering a lag of $L,M \in \mathbb{N}_{\ge 0}$ for the features and the target respectively, the Bayes error suffered by considering only the set of features $X_{\Bar{A}}$ is bounded by:
{\fontsize{7.5}{8}
\begin{equation}\label{eq:classErr}
    \begin{aligned}
    \inf_{g\in \mathcal{G}_{\abar}}\E_{X,Y}\left[\mathbb{1}_{Y_t\neq g\left(X_{\abar}^{\xtpast},Y_{\ytpast}\right)}\right] \leq \epsilon + \sqrt{2\cdot TE_{X_{A}\rightarrow Y| X_{\abar}}},
    \end{aligned}
\end{equation}}
where {\fontsize{8.5}{9}$\epsilon = \E_{X,Y}\left[\mathbb{1}_{Y_t\neq \arg\max_{y_t\in\mathcal{Y}}p(y_t|X^{\xtpast},Y_{\ytpast})}\right]$} is the irreducible Bayes error and $\mathbb{1}$ is the indicator function.
\end{theorem}


The presented results enable the derivation of algorithms based on conditional TE to perform \emph{causal} feature selection in a sound way. Indeed, we can easily control the loss in prediction performance via the conditional transfer entropy of the discarded features, for both regression and classification.


\subsection{Algorithms}\label{sec:TEFSalgo}
Building on the presented results, in this section, we derive \emph{Backward} and \emph{Forward} \emph{Transfer Entropy Feature Selection} (TEFS) for causal feature selection via conditional TE. For both of them, we provide guarantees on the regression/classification error at the end of the selection process.

\paragraph{Backward Approach}
Algorithm \ref{alg:TEFSalgoBack} reports the pseudo-code of the \emph{Backward TEFS} algorithm, defined in terms of the feature and target lags $M$ and $L$ and in terms of the maximum information loss $\delta$.
The algorithm begins by assuming that all features are significant to the target variable (line~\ref{line:3}). Then, it iteratively selects the feature with the smallest information flow, given the remaining features (line~\ref{line:7}), and removes it if the loss of information flow keeps the error below $\delta$ (lines~\ref{line:9}-\ref{line:10}). When removing a feature leads to a cumulative information loss greater than the maximum desired (line~\ref{line:6}), Backward TEFS stops and returns the remaining features (line~\ref{line:12}). These are the candidate \emph{causal} drivers of the current value of the target in terms of conditional TE. 

\begin{remark}[About the hyperparameters $M$, $L$, and $\delta$]\label{rem:hyper}
The two hyperparameters $M$ and $L$ determine the number of past time steps that the algorithm is considering to evaluate the importance of the features on the current target and filter out the autoregressive component. They balance the positivity-unconfoundedness tradeoff (and, more generally, the bias-variance tradeoff in ML). Indeed, considering a larger lag for a covariate may lead to a better understanding of its impact on the actual target and increases the adherence to the causal sufficiency assumption. 
On the other hand, given a time series of length $T$, increasing the number of past values of the features (and of the target) 
reduces the number of available samples (at the limit case, considering $L=M=T$ we only have one sample with the last observed target $Y_T$ and all past features and target to condition on). This may lead to overfitting and the curse of dimensionality~\citep{damour2020}. Finally, the hyperparameter $\delta$ is the maximum information loss, in terms of TE, that the user is willing to lose in order to reduce the number of features, as theoretically justified by the next theorem.
\end{remark}

\begin{algorithm}[t]
\caption{Backward TEFS: Backward Transfer Entropy Feature Selection}\label{alg:TEFSalgoBack}
\small
\begin{algorithmic}[1]
\Require{$D$ features with $L$ past values, $\{X_1^{\xtpast},\dots,X_D^{\xtpast}\}$; target $Y_t$ and its $M$ past values $Y_{\ytpast}$; $N$ samples; maximum information loss $\delta$}
\Ensure{set of selected features $X_{\abar}$}\vspace{.2cm}
\State{$X_{{A}} \leftarrow \{\}$}\Comment{Initialization: no removed features}
\State{$X_{{\abar}} \leftarrow \{X_1^{\xtpast},\dots,X_D^{\xtpast}\}$}\label{line:3} 
\State{$TE\_loss\leftarrow0$}
\vspace{.2cm}
\While {$TE\_loss\leq \frac{\delta}{2B^2}$}\label{line:6}\\
\Comment{$TE\_loss\leq \frac{\delta^2}{2}$ for classification}
\State{$X_{\hat{i}}\leftarrow\arg\min_{X_i\in X_{\abar}} TE_{X_{i}\rightarrow Y|X_{\abar}\setminus X_{i}} $}\label{line:7} 
\State{$TE\_loss \leftarrow TE\_loss+TE_{X_{\hat{i}}\rightarrow Y|X_{\abar}\setminus X_{\hat{i}}}$}
\State{$X_{{A}} \leftarrow X_{{A}} \cup X_{\hat{i}}$}\label{line:9} 
\State{$X_{{\abar}} \leftarrow X_{{\abar}} \setminus X_{\hat{i}}$}\label{line:10} 
\EndWhile   \\ 
\Return{$X_{{\abar}}$}\label{line:12}
\end{algorithmic}
\end{algorithm}

\begin{remark}[About temporal lags and features]\label{rem:variant}
By definition of TE, the two algorithms evaluate the contribution of the last $L$ values $X_i^{\xtpast}$ of any feature $i$ as a whole. One may consider each value $X_i^{t-\tau},\tau\in\{1,\dots,L\}$ of each feature at a different temporal lag as a different feature and if needed remove each one separately. Although this would be a forcing of the concept of TE, it would lead to a more accurate identification of the causal interaction between the features and the target, since each of the remaining features would identify a specific feature and time lag, similar to the PCMCI approach. However, this would lead to a larger feature space, an increased computational complexity and a larger risk of the curse of dimensionality. For this reason, in line with the definition of conditional TE and Granger causality, the proposed approach considers the contribution of a specific feature at different time lags as a whole. 
\end{remark}

\begin{algorithm}[t]
\caption{Forward TEFS: Forward Transfer Entropy Feature Selection}\label{alg:TEFSalgoForw}
\small
\begin{algorithmic}[1]
\Require{$D$ features with $L$ past values, $\{X_1^{\xtpast},\dots,X_D^{\xtpast}\}$; target $Y_t$ and its $M$ past values $Y_{\ytpast}$; $N$ samples; minimum information gain $\Delta$}
\Ensure{set of selected features $X_{A}$}\vspace{.2cm}
\State{$X_{{A}} \leftarrow \{\}$}\Comment{Initialization: no selected features}\label{line:2_1}
\State{$X_{{\abar}} \leftarrow \{X_1^{\xtpast},\dots,X_D^{\xtpast}\}$} \label{line:2_2}
\State{$TE\_cumulated\leftarrow0$}
\vspace{.2cm}
\While {$TE\_cumulated\leq \frac{\Delta}{2B^2}$}\label{line:2_4}\\
\Comment{$TE\_cumulated\leq \frac{\Delta^2}{2}$ for classification}
\State{$X_{\hat{i}}\leftarrow\arg\max_{X_i\in X_{\abar}} TE_{X_{i}\rightarrow Y|X_{A}} $}\label{line:2_3}
\State{$TE\_cumulated \leftarrow TE\_cumulated+TE_{X_{\hat{i}}\rightarrow Y|X_{A}}$}
\State{$X_{{A}} \leftarrow X_{{A}} \cup X_{\hat{i}}$}
\State{$X_{{\abar}} \leftarrow X_{{\abar}} \setminus X_{\hat{i}}$}
\EndWhile   \\ 
\Return{$X_{{A}}$}
\end{algorithmic}
\end{algorithm}

The following lemma proves that the overall information loss by removing the set of features $X_A$ can be computed as the cumulative transfer entropy loss evaluated at each step. 

\begin{lemma}\label{thm:recursiveBackward}
Let $T(X_A)$ be the conditional TE between the target and the discarded features $X_A$, given its complementary $X_{\abar}$ 
    $(T(X_A) \coloneqq TE_{X_{A}\rightarrow Y|X_{\abar}})$. 
Let also $T_{\abar}=\min_{X_i\in X_{\abar}} TE_{X_{i}\rightarrow Y|X_{\abar}\setminus X_{i}}$ be the minimum conditional TE of a single feature among the selected ones and the target, given the other selected features, and $X_{\hat{i}}=\arg\min_{X_i\in X_{\abar}} TE_{X_{i}\rightarrow Y|X_{\abar}\setminus X_{i}}$ its argument. Then it holds:
\begin{equation}\label{eq:recursiveBackward}
    T(X_A\cup X_{\hat{i}}) = T(X_A) + T_{\abar}.
\end{equation}
\end{lemma}

Lemma \ref{thm:recursiveBackward} can be exploited to justify the choice of the stopping condition shown in Algorithm \ref{alg:TEFSalgoBack}. Indeed the following theorem shows that, by considering the proposed stopping criterion, the increase of the ideal regression (or classification) error is bounded by the maximum information loss $\delta$.

\begin{theorem}\label{thm:backError}
    The regression error suffered considering the variables $X_{\abar}$ selected by \emph{Backward TEFS} is bounded by: 
    \begin{equation*}
        \inf_{g\in \mathcal{G}_{\abar}}\E_{X,Y}[(Y_t-g(X_{\abar}^{\xtpast},Y_{\ytpast}))^2] \leq \sigma^2+\delta.
    \end{equation*}
    Similarly, the classification error is bounded by:
    \begin{equation*}
        \inf_{g\in \mathcal{G}_{\abar}}\E_{X,Y}\left[\mathbb{1}_{Y_t\neq g(X_{\abar}^{\xtpast},Y_{\ytpast})}\right] \leq \epsilon + \delta.
      \end{equation*}
\end{theorem}
\begin{proof}
    At the first iteration of the algorithm $T(X_A)=T(\emptyset)=0$. Iteratively applying Equation \eqref{eq:recursiveBackward}, at the end of the algorithm we have \fontsize{8.5}{9}{$TE_{X_{A}\rightarrow Y| X_{\abar}}=\sum_{k=1}^K TE_{X_{k}\rightarrow Y|X_{\abar_k}\setminus X_{k}}$}, where $K$ is the number of iterations and $X_k,X_{\abar_k}$ are the removed feature and the set of candidate selected features during the $k$-th iteration. Since the algorithm stops when this quantity reaches the threshold (line \ref{line:6}), at the end we have $TE_{X_{A}\rightarrow Y| X_{\abar}}\leq \frac{\delta}{2B^2}$($\frac{\delta^2}{2}$ in classification). Finally, starting from Equation \eqref{eq:regrErr} (Equation \eqref{eq:classErr} in classification), substituting this upper bound on $TE_{X_{A}\rightarrow Y| X_{\abar}}$ the result follows. 
\end{proof}


\paragraph{Forward Approach}
The pseudo-code of \emph{Forward TEFS} can be found in Algorithm \ref{alg:TEFSalgoForw},  defined in terms of the feature and target lags $M$ and $L$ and in terms of the minimum information gain $\Delta$. As in the backward case, $X_{A}$ is initialised to the empty set, and all features are in the complementary set $X_{\abar}$ (lines \ref{line:2_1}-\ref{line:2_2}). In this forward case, however, the selected features are represented by the set $X_A$. At each iteration, the algorithm selects the feature, among those not selected, that maximises the TE with the target, conditioned to the already selected features (line \ref{line:2_3}). 
Finally, the stopping condition ensures that the procedure ends when the amount of information flowing from the features to the target is \emph{sufficiently large}, depending on the hyperparameter $\Delta$ (line \ref{line:2_4}). A small value of the minimum information gain $\Delta$ leads to the selection of fewer features and a greater loss, a large value induces the selection of many features and a smaller loss of information. Furthermore, the hyperparameters $L,M$ regulate the bias-variance tradeoff as discussed in Remark \ref{rem:hyper}. 

The following lemma proves that the overall information obtained with the \emph{Forward TEFS}, which selects the set $X_A$, is the cumulative conditional TE maximised at each step.

\begin{lemma}\label{thm:TEaddingAFeature}
    Let $X_i\in X_{\abar}$ be a feature in the complementary set of $X_A$. Then, the amount of information that flows from $X_A$ and the feature $X_i$ to the target $Y$ in terms of TE is:
    \begin{equation}\label{eq:TEaddingAFeature}
        TE_{X_{A}\cup X_i \rightarrow Y} = TE_{X_{A}\rightarrow Y} + TE_{X_{i}\rightarrow Y|X_{A}}.
    \end{equation}
\end{lemma}
Intuitively, Lemma \ref{thm:TEaddingAFeature} shows that the information flowing from $X_A$ and the feature $X_i$ to the target $Y$ is the information flowing from $X_A$ to the target plus the additional information flowing from $X_i$ to the target, given the information already flowing from the features in $X_A$.

Finally, in the following theorem we derive an upper bound on the ideal error, which identifies in the hyperparameter $\Delta$ the way to regulate the minimum information gain desired in the selection procedure of Algorithm \ref{alg:TEFSalgoForw}.

\begin{theorem}\label{thm:forwError}
    The regression error suffered considering the variables $X_{A}$ selected by \emph{Forward TEFS} is bounded by: 
    {\fontsize{7.5}{8}\selectfont
    \begin{equation*}
        \inf_{g\in \mathcal{G}_{A}}\E_{X,Y}[(Y_t-g(X_{A}^{\xtpast},Y_{\ytpast}))^2] \leq \sigma^2+2B^2 \cdot TE_{X\rightarrow Y}-\Delta.
    \end{equation*}}
    Similarly, the classification error is bounded by:
    {\fontsize{7.5}{8}\selectfont
    \begin{equation*}
    \inf_{g\in \mathcal{G}_{A}}\E_{X,Y}\left[\mathbb{1}_{Y_t\neq g(X_{A}^{\xtpast},Y_{\ytpast})}\right] \leq \epsilon + \sqrt{2\cdot TE_{X\rightarrow Y}-\Delta^2}.
    \end{equation*}}
\end{theorem}

\begin{proof}
Similarly to Theorem \ref{thm:backError}, the proof exploits the recursive relationship of Lemma \ref{thm:TEaddingAFeature} to bound the final amount of information with the hyperparameter $\Delta$. Then, reformulating the upper bound on the ideal error of Theorem \ref{thm:regrErr} and \ref{thm:classErr}, the results follow. Details can be found in Appendix \ref{app:TEFS}.
\end{proof}



\section{Transfer Entropy Estimation and Finite-Sample Analysis}\label{sec:TEinv}
In this section, we extend the theoretical analysis by discussing the effect of estimating the TE with a finite number of samples. We provide an upper bound on the regression and classification error that holds in high probability with a specific kernel estimator of the TE.

\paragraph{TE Estimation}
As discussed, TE is a specific instance of a CMI, whose finite-sample estimation is a challenging task. Existing approaches rely on kernel density or $k$-nearest neighbour estimators. A traditional $k$-nearest neighbour estimator for continuous variables has been proposed in~\citep{kraskov2004}, and it has been extended in more recent works~\citep{gao2017estimating,runge2018conditional}. The usual theoretical guarantee of these approaches is asymptotic consistency, while a finite-sample analysis is little studied and remains largely an open problem. Considering continuous features and a continuous target, ~\citep{singh2014generalized,singh2014exponential} show a finite-sample exponential concentration inequality for density functional estimators and extend this theoretical result to conditional density functionals, respectively. In our analysis, we consider the \emph{mirrored kernel density estimator} proposed in Section 4 of~\citep{singh2014generalized}. It first estimates probability density functions through a modified kernel density estimator. 
Then, the final estimator is obtained by plugging the estimated density into the functional of interest (i.e., CMI for our purposes).
The resulting CMI estimate {\fontsize{8.5}{9}$\hat{I}(X;Y|Z)$}, made with $n$ samples, enjoys the following concentration bound with probability at least $1-\eta$ (Section 6.2 of ~\citet{singh2014exponential}): 
\begin{equation}\label{eq:CMIconcentration}
    |I(X;Y|Z)-\hat{I}(X;Y|Z)| \leq \underbrace{\sqrt{\frac{4C_V^2}{n}\log\left(\frac{2}{\eta}\right)}}_{\text{(i)}}+\underbrace{C_{B}n^{\frac{-\beta}{\beta+d}}}_{\text{(ii)}}, 
\end{equation}

where $d=d_x+d_y+d_z$ is the sum of the dimensions of the three vectors $X$, $Y$, and $Z$. The terms $C_V,C_B>0$ are constant \wrt the number of sample $n$ but exponentially dependent on the dimension $d$ of the variables.\footnote{Intuitively, high dimensional vectors lead to larger uncertainty on the estimator, which may suffer from curse of dimensionality.}
$\beta\in [0,1]$ is the order of the H\"older space to which the joint density $p(x,y,z)$ belongs (e.g., $\beta=1$ implies Lipschitz continuity).\footnote{We refer the interested reader to Section 2 of~\citep{singh2014generalized} for a detailed presentation.}  
Intuitively, the term (i) in  Equation \eqref{eq:CMIconcentration} accounts for the variance of the density estimate and converges to zero with the optimal rate $\mathcal{O}(n^{-\frac{1}{2}})$. Instead, the term (ii) refers to the bias of the density estimator and displays a convergence rate of order  $\mathcal{O}(n^{\frac{-\beta}{\beta+d}})$, which is therefore the overall convergence rate.


\paragraph{Finite-Sample Bound}
Recalling Theorem \ref{thm:backError} and \ref{thm:forwError} in the ideal case (i.e., when the TE is assumed to be exact), the following result rephrases them in the finite-sample case (i.e., when the TE is estimated from the kernel density estimator discussed above). The proof can be found in Appendix \ref{app:TEFS}.

\begin{figure*}[t]
     \centering
     \begin{subfigure}{0.15\textwidth}
         \centering
        \includegraphics[width=\textwidth]{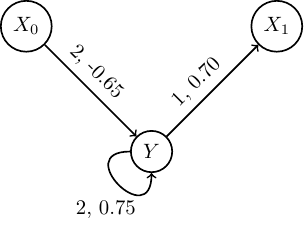}
         \caption{Three variables.}
         \label{fig:graph3}
     \end{subfigure}
     \hfill
     \begin{subfigure}{0.2\textwidth}
         \centering
         \includegraphics[width=\textwidth]{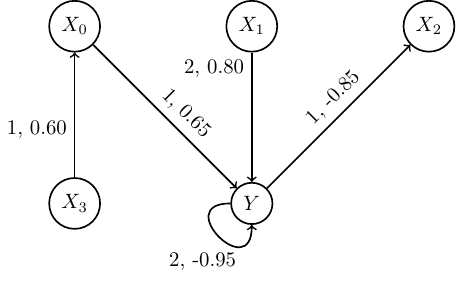}
         \caption{Five variables.}
         \label{fig:graph5}
     \end{subfigure}
     \hfill
     \begin{subfigure}{0.35\textwidth}
         \centering
        \includegraphics[width=\textwidth]{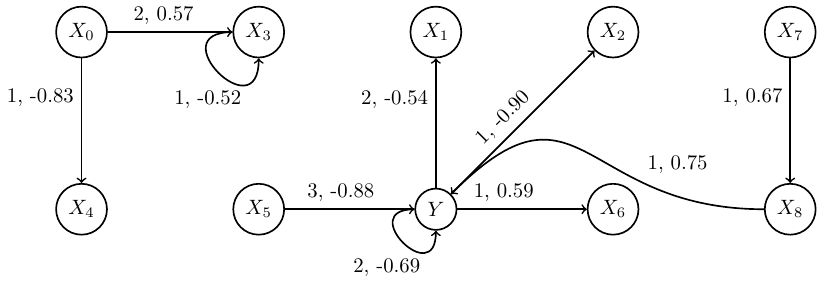}
         \caption{Ten variables.}
         \label{fig:graph10}
     \end{subfigure}
     \hfill
        \caption{Causal graphs considered for synthetic experiments. Each edge reports the time lag and the interaction coefficient.}
        \label{fig:linearSynth}
\end{figure*}

\begin{theorem}\label{corollary:finiteSample}
The regression error suffered considering the variables $X_{\abar}$ selected by Backward TEFS with the estimator of CMI proposed in ~\citep{singh2014exponential}, with $n$ samples, is bounded, with probability at least $1-\eta$, by:
\begin{align*}
    &\inf_{g\in \mathcal{G}_{\abar}}\E_{X,Y}[(Y_t-g(X_{\abar}^{\xtpast},Y_{\ytpast}))^2] \\ &\leq \sigma^2+\delta+2B^2K\left( \sqrt{\frac{4C_V^2}{n}\log\left(\frac{2^{d+1}}{\eta}\right)}+C_{B}n^{\frac{-\beta}{\beta+d}}\right) ,
\end{align*}
where $K$ is the number of iterations performed by the algorithm and $d$ is the dimension of the full set of variables ($d=LD+M+1$).
The regression error suffered considering the selected variables $X_{A}$ by running Forward TEFS, under the same conditions,  is bounded, with probability at least $1-\eta$, by:
\begin{align*}
\begin{split}
        &\inf_{g\in \mathcal{G}_{A}}\E_{X,Y}[(Y_t-g(X_{A}^{\xtpast},Y_{\ytpast}))^2] \\&\leq \sigma^2+2B^2 \cdot TE_{X\rightarrow Y}-\Delta\\& +2B^2K\left( \sqrt{\frac{4C_V^2}{n}\log\left(\frac{2^{d+1}}{\eta}\right)}+C_{B}n^{\frac{-\beta}{\beta+d}} \right).
        \end{split}
    \end{align*}
\end{theorem}

    Analogous results hold in classification, extending the ideal Bayes error of Theorem \ref{thm:backError}, \ref{thm:forwError} in the finite-sample case.
\begin{remark}[About the choice of $M$ and $L$ in a finite-sample setting]
       
\end{remark}
  

\section{Experimental Validation}\label{sec:experiments}

In this section, we conduct an experimental validation of the proposed algorithms on both synthetic (Section~\ref{subsec:synth}) and real-world (Section~\ref{sec:realWorld}) domains.\footnote{Code and datasets can be found at: \url{https://www.dropbox.com/scl/fo/ctvxqh8q131f8sim4s22o/h?dl=0&rlkey=5hsbpold82bi8vilu44ubb508}.}

\subsection{Synthetic Experiments}\label{subsec:synth}
We analyse the behaviour of the forward and backward TEFS \wrt the forward and backward feature selection algorithms based on CMI proposed in~\citep{Beraha2019} and with the PCMCI algorithm~\citep{runge2019detecting}, a state-of-the-art causal discovery approach. Details and additional results can be found in Appendix \ref{app:experiments}.

Figure \ref{fig:linearSynth} shows the graph built to perform the experiments in three, five, and ten dimensions, with the time lag of the interaction between each pair of variables and the coefficient of their linear interaction. Indeed, the structural causal model of each variable is assumed to be linear, with additive Gaussian noise. Each of the three experiments is repeated ten times with different seeds, considering the coefficients reported in the figure, $N=300$ samples, maximum time lags $M=L=2$ for the two and five-dimensional data, and $M=L=3$ for the ten-dimensional data, and additive Gaussian noise with standard deviation $s=0.1$. Then, the individual effect of each of these parameters is studied, fixing the other parameters. To further inspect the impact of high dimensionality, we also consider $D\in\{15,20,40,60,80,100\}$, preserving the relationships and coefficients as in Figure \ref{fig:graph10} and adding independent variables that can be autocorrelated or interact with each other. 

Each model is evaluated with the true positive rate (TPR) and false positive rate (FPR) on the number of causal links detected, leading to some considerations on the algorithms. Both forward and backward CMI feature selection do not consider the target in the conditioning. Hence they tend to select features that are caused by the target as well. On the contrary, forward and backward TEFS and the PCMCI algorithm tend to identify the features causally relevant to the target, along with its autoregressive component. Additionally, PCMCI tends to consider more features to avoid false negatives at the cost of a higher false positive rate. In contrast, the two versions of the TEFS algorithm are more focused on identifying the causal features that are most significant in terms of regression performance. This translates into better control of the FPR, at the cost of possibly not considering some less relevant causal links, worsening the TPR. Table \ref{tab:synth3D} shows, for the three-dimensional setting, the complete results in terms of TPR and FPR,
In Appendix \ref{app:experiments}, it is possible to find the same results for a larger number of features $D\in\{5,10,15,20,40,60,80,100\}$. 

\begin{table*}[ht]
\caption{Three-dimensional experiments on synthetic datasets. Each experiment has been repeated ten times, with different seeds, varying one parameter at a time.
\label{tab:synth3D}}
\fontsize{8.5pt}{8.5pt}\selectfont
\centering 
\begin{tabular}{@{}cccccc@{}} \hline 
\ & Forward CMI & Backward CMI & PCMCI & \cellcolor{lightgray!25} Forward TEFS & \cellcolor{lightgray!25} Backward TEFS \\\hline 
\textbf{Benchmark} & & & & \cellcolor{lightgray!25} & \cellcolor{lightgray!25} \\
TPR & $0.5$ & $0.5$ & $1.0$ & \cellcolor{lightgray!25} $1.0$ & \cellcolor{lightgray!25} $1.0$\\
FPR & $1$ & $1$ & $0.1$ & \cellcolor{lightgray!25} $0$ & \cellcolor{lightgray!25} $0$ \\
\hline
\textbf{Increasing Noise} & & & & \cellcolor{lightgray!25} & \cellcolor{lightgray!25} \\
TPR & $[0.5, 0.5, 0.5, 0.5, 0.45]$ & $[0.5, 0.5, 0.5, 0.5, 0.45]$  & $[1, 1, 1, 1, 1]$ & \cellcolor{lightgray!25} $[1, 1, 1, 0.95, 0.9]$ & \cellcolor{lightgray!25} $[1, 1, 1, 1, 1]$\\
FPR & $[1,1,1,1,1]$ & $[1,1,1,1,1]$ & $[0.1, 0, 0.1, 0.3, 0.4]$ & \cellcolor{lightgray!25} $[0, 0, 0, 0, 0]$ & \cellcolor{lightgray!25} $[0, 0, 0, 0, 0]$ \\
\hline
\textbf{Increasing $\tau$} & & & & \cellcolor{lightgray!25} & \cellcolor{lightgray!25} \\
TPR & $[0.5, 0.5, 0.5, 0.5]$ & $[0.5, 0.5, 0.5, 0.5]$  & $[1, 1, 1, 1]$ & \cellcolor{lightgray!25} $[1, 1, 1, 1]$ & \cellcolor{lightgray!25} $[1, 1, 1, 1]$\\
FPR & $[1,1,1,1]$ & $[1,1,1,1]$ & $[0.1, 0, 0.3, 0.3]$ & \cellcolor{lightgray!25} $[0, 0, 0, 0]$ & \cellcolor{lightgray!25} $[0, 0, 0, 0]$ \\
\hline
\textbf{Increasing N} & & & & \cellcolor{lightgray!25} & \cellcolor{lightgray!25} \\
TPR & $[0.5, 0.5, 0.5, 0.5, 0.5]$ & $[0.5, 0.5, 0.5, 0.5, 0.5]$ & $[0.9, 1, 1, 1, 1]$ & \cellcolor{lightgray!25} $[1, 1, 1, 1, 1]$ & \cellcolor{lightgray!25} $[1, 1, 1, 1, 1]$\\
FPR & $[1,1,1,1,1]$ & $[1,1,1,1,1]$ & $[0.5, 0.2, 0.1, 0.1, 0.1]$ & \cellcolor{lightgray!25} $[0, 0, 0, 0, 0]$ & \cellcolor{lightgray!25} $[0, 0, 0, 0, 0]$ \\
\hline
\textbf{Random Coefficients} & & & & \cellcolor{lightgray!25} & \cellcolor{lightgray!25} \\
TPR & $[0.5,0.5,0.5,0.5,0.3]$ & $[0.5,0.5,0.5,0.5,0.3]$ & $[1,1,1,0.75,1]$ & \cellcolor{lightgray!25} $[1,1,1,1,1]$ & \cellcolor{lightgray!25} $[1,1,1,1,1]$\\
FPR & $[1, 1, 1, 1, 1]$ & $[1, 1, 1, 1, 1]$ & $[0.2,0.2,0.1,0.1,0.2]$ & \cellcolor{lightgray!25} $[0,0,0,0,0.1]$ & \cellcolor{lightgray!25} $[0,0,0,0,0.1]$ \\
\hline
\end{tabular}
\end{table*}

\begin{table*}[t]
\fontsize{8.5}{8.5}\selectfont
\caption{Real-world experiments for three different datasets. Each experiment has been repeated three times, considering different time-depths $M=L\in\{1,2,3\}$. The table reports the set of IDs of the selected features and the related $R^2$ test score.
\label{tab:realRed}}
\centering
\begin{tabular}{@{}cccccc@{}} \hline 
\ & Best\_CMI & Best\_CI & \cellcolor{lightgray!25} Forward TEFS  & \cellcolor{lightgray!25} Backward TEFS \\\hline 
\textbf{Climate, 5 features} & & & \cellcolor{lightgray!25} & \cellcolor{lightgray!25} \\
$M=L=1$ & $\{0,3,4\}\rightarrow0.34$ & $\{0,2,4\}\rightarrow0.45$ & \cellcolor{lightgray!25} $\{0,4\}\rightarrow0.44$ & \cellcolor{lightgray!25} $\{0,4\}\rightarrow0.44$\\
$M=L=2$ & $\{0,1\}\rightarrow0.36$ & $\{0,1,3,4\}\rightarrow0.47$ & \cellcolor{lightgray!25} $\{0,4\}\rightarrow0.42$ & \cellcolor{lightgray!25} $\{0,4\}\rightarrow0.42$ \\
$M=L=3$ & $\{0,1\}\rightarrow0.35$ & $\{0,3,4\}\rightarrow0.48$ & \cellcolor{lightgray!25} $\{0,4\}\rightarrow0.47$ & \cellcolor{lightgray!25} $\{0,4\}\rightarrow0.47$ \\
\hline
\textbf{Climate, 15 features} & & & \cellcolor{lightgray!25} & \cellcolor{lightgray!25} \\
$M=L=1$ & $\{0,2,3,11\}\rightarrow0.29$ & $\{0,2,7\}\rightarrow0.45$ & \cellcolor{lightgray!25} $\{0,7\}\rightarrow0.44$ & \cellcolor{lightgray!25} $\{0,11\}\rightarrow0.44$\\
$M=L=2$ & $\{0,2,11\}\rightarrow0.30$ & $\{5,7,11\}\rightarrow0.47$ & \cellcolor{lightgray!25} $\{0,7\}\rightarrow0.43$ & \cellcolor{lightgray!25} $\{0,11\}\rightarrow0.42$ \\
$M=L=3$ & $\{0,5\}\rightarrow0.28$ & $\{2,7\}\rightarrow0.48$ & \cellcolor{lightgray!25} $\{7\}\rightarrow0.47$ & \cellcolor{lightgray!25} $\{11\}\rightarrow0.43$ \\
\hline
\textbf{Benchmark, 5 features} & & & \cellcolor{lightgray!25} & \cellcolor{lightgray!25} \\
$M=L=1$ & $\{3,4\}\rightarrow0.09$ & $\{3,4\}\rightarrow0.09$ & \cellcolor{lightgray!25} $\{3,4\}\rightarrow0.09$ & \cellcolor{lightgray!25} $\{3,4\}\rightarrow0.09$\\
$M=L=2$ & $\{3\}\rightarrow0.07$ & $\{3,4\}\rightarrow0.13$ & \cellcolor{lightgray!25} $\{4\}\rightarrow0.09$ & \cellcolor{lightgray!25} $\{4\}\rightarrow0.09$\\
$M=L=3$ & $\{0,3\}\rightarrow0.07$ & $\{0,3,4\}\rightarrow0.13$ & \cellcolor{lightgray!25} $\{3\}\rightarrow0.11$ & \cellcolor{lightgray!25} $\{3\}\rightarrow0.11$\\
\hline
\end{tabular}
\end{table*}

\subsection{Real-World Experiments}\label{sec:realWorld}
We conclude the validation with three applications to real-world data. In particular, we firstly analyse a climatic datasets where the target represents the vegetation state of a basin of the Po River, and the five features are meteorological variables related to temperature and precipitation. To increase the dimensionality of the data, we also consider a similar dataset with fifteen meteorological features. Finally, to further test the algorithm on a benchmark dataset, we consider a dataset composed of five variables available in an online causal repository (\url{https://causeme.uv.es/}), where we consider in turn each variable to be the target and the other four variables the candidate causal features.
Ideally, we would know the real underlying causal relationships to validate the algorithms, which is unrealistic in complex problems. For this reason, motivated by the feature selection inspiration of this work, we propose the regression performance (in terms of coefficient of determination $R^2$) to draw some conclusions on the effectiveness of the methods. To further inspect the performance of the proposed approaches, along with forward and backward feature selection based on CMI, we consider as causal discovery baselines PCMCI and FullCI, which are implemented in the same library (\url{https://github.com/jakobrunge/tigramite.git}). These two methods are applied considering both linear and non-linear independence tests. Table \ref{tab:realRed} shows the selected set of features and the associated $R^2$ test score, considering increasing past depth $L=M\in\{1,2,3\}$. For clarity, only the best-performing result among the linear and non-linear PCMCI and FullCI applications (\emph{best\_CI}) is reported. The same applies to the forward and backward CMI approaches (\emph{best\_CMI}). Furthermore, the three datasets in the table are one climatic target with five features, the climatic target with fifteen features, and the benchmark dataset with five variables, where the third variable is taken as the target. The complete results for all the considered datasets can be found in Appendix \ref{app:experiments}. As discussed for the synthetic experiments, from the results, it is possible to observe the tendency of the \emph{Forward and Backward TEFS}, \wrt PCMCI and FullCI, to consider a reduced set of the most related causal features in terms of TE. In the empirical results, this leads to an empirical performance that preserves most of the information while significantly reducing dimensionality. Therefore, the $R^2$ scores of the proposed algorithms are similar to or slightly worse than the best causal discovery counterparts, which consider more features as they aim to identify all possible causal links.
More extensive applications of \emph{Forward and Backward TEFS} on complex real data are the subject of future research.

\section{Conclusions and Future Developments}\label{sec:conclusions}
In this paper, we introduced two novel causal feature selection approaches for time series, \emph{Forward} and \emph{Backward TEFS}, relying on conditional TE to guide the selection. We provided an analysis of the information loss due to the selection procedure in the exact and in the finite-sample cases. Our theoretical results validate the soundness of the two approaches, formally highlighting the role of the hyperparameters in controlling the final selection.
Then, we conducted synthetic experiments in comparison with standard CMI feature selection and the state-of-the-art PCMCI, highlighting the abilities of our method to successfully uncover causal relationships.
Finally, we performed experiments on real-world datasets (where no ground truth of causal relationships is present), proving that our methods display competitive results in terms of regression scores.
A promising future development is related to the \emph{invariance} framework, which relies on the intuition that a variable causing the target maintains the same relationship with the target itself, even if its
own distribution changes~\citep{arjovsky2019invariant,buhlmann2020invariance}. The concentration bound of Equation \eqref{eq:CMIconcentration} could be therefore exploited to test the capability of the selected features to generalize across heterogeneous environments.

\bibliography{aaai24}

\clearpage

\appendix

\section{Transfer Entropy for Causal Feature Selection: additional proofs and results}\label{app:TEFS}
This section contains proofs and additional results of Section \ref{sec:TEFS} and \ref{sec:TEinv}. To simplify the notation, $X_A^{-t},X_{\abar}^{-t}$ will denote the values of the set of features with indices in $A$ or its complementary $\abar$ in the timesteps $\xtpast$. Moreover, $Y_{-t}$ will denote the last values of the target $Y_{\ytpast}$ and $Y_t$ will be the current value of the target itself.
\begin{proof}[Proof of Theorem \ref{thm:regrErr}]
Since the best regression model, given the inputs, is the conditional expected value of the target, the following equality holds:

\begin{equation*}
\begin{aligned}
    \inf_{g\in \mathcal{G}_{\abar}}\E_{X,Y}[(Y_t-g(X_{\abar}^{-t},Y_{-t}))^2] \\= \E_{X,Y}[(Y_t-\E[Y_t|X_{\abar}^{-t},Y_{-t}])^2].
\end{aligned}
\end{equation*}

Adding and subtracting $\E[Y_t|X^{-t},Y_{-t}]$ to the inner term and writing the outer expected value as integral, the term above is equal to:
{\fontsize{6.5}{7.5}
\begin{equation*}
\begin{aligned}
     &\int p(x^{-t},y_{-t}) \int p(y_t|x^{-t},y_{-t}) \big( y_t-\E[Y_t|X_{\abar}^{-t},Y_{-t}] \\ &\pm \E[Y_t|X^{-t},Y_{-t}] \big)^2 dy_tdx^{-t}dy_{-t}\\
    =& \int p(x^{-t},y_{-t}) \int p(y_t|x^{-t},y_{-t}) \big( y_t-\E[Y_t|X^{-t},Y_{-t}] \big)^2 dy_tdx^{-t}dy_{-t}\\
    +&\int p(x^{-t},y_{-t}) \int p(y_t|x^{-t},y_{-t}) \left(\E[Y_t|X^{-t},Y_{-t}]-\E[Y_t|X_{\abar}^{-t},Y_{-t}] \right)^2\\ 
    +&2\E_{X,Y}[( y_t-\E[Y_t|X^{-t},Y_{-t}] )\cdot (\E[Y_t|X^{-t},Y_{-t}]-\E[Y_t|X_{\abar}^{-t},Y_{-t}] )].
    \end{aligned}
\end{equation*}
}
We now analyse each term of the right-hand side of this equation. 

The first term is the irreducible mean squared error:
{\fontsize{6.5}{7}
\begin{equation*}
\begin{aligned}
    \int p(x^{-t},y_{-t}) & \int p(y_t|x^{-t},y_{-t}) \left( y_t-\E[Y_t|X^{-t},Y_{-t}] \right)^2 dy_tdx^{-t}dy_{-t} \\
    &=\E_{X,Y}[(Y_t-\E[Y_t|X^{-t},Y_{-t})^2] = \sigma^2.
    \end{aligned}
\end{equation*}
}
The third term is equal to zero:
{\fontsize{7.5}{8}
\begin{equation*}
\begin{aligned}
    \E_{X,Y}[ &( Y_t-\E[Y_t|X^{-t},Y_{-t}] )\cdot (\E[Y_t|X^{-t},Y_{-t}]-\E[Y_t|X_{\abar}^{-t},Y_{-t}] )]\\
    =&\E_{X,Y_{-t}}\big[\E_{Y_t}[( Y_t-\E[Y_t|X^{-t},Y_{-t}] ) \\ &\cdot (\E[Y_t|X^{-t},Y_{-t}]-\E[Y_t|X_{\abar}^{-t},Y_{-t}] )]|X^{-t},Y_{-t}\big]\\
    =&\E_{X,Y_{-t}}[( \E_{Y_t}[Y_t|X^{-t},Y_{-t}]-\E[Y_t|X^{-t},Y_{-t}] )\\&\cdot (\E[Y_t|X^{-t},Y_{-t}]-\E[Y_t|X_{\abar}^{-t},Y_{-t}] )]=0,
    \end{aligned}
\end{equation*}
}
where the second equality holds since only the first term ($Y_t$) depends on $Y_t$ itself.

Finally, the second term can be bounded as:
{\fontsize{6.5}{7}
    \begin{align*}
        & \int p(x^{-t},y_{-t}) \int p(y_t|x^{-t},y_{-t}) \\&\left(\E[Y_t|X^{-t},Y_{-t}]-\E[Y_t|X_{\abar}^{-t},Y_{-t}] \right)^2 dy_tdx^{-t}dy_{-t}\\
        = &  \int p(x^{-t},y_{-t}) \int p(y_t|x^{-t},y_{-t}) \\&\left(\int y_t\cdot (p(y_t|x^{-t},y_{-t})-p(y_t|x_{\abar}^{-t},y_{-t}))dy_t\right)^2 dy_tdx^{-t}dy_{-t}\\
        \leq & B^2 \int p(x^{-t},y_{-t}) \\&\left(\int |p(y_t|x^{-t},y_{-t})-p(y_t|x_{\abar}^{-t},y_{-t})|dy_t\right)^2 dx^{-t}dy_{-t}\\
        = & B^2 \int p(x^{-t},y_{-t}) \\& \left(2\cdot D_{TV}(p(y_t|x^{-t},y_{-t})||p(y_t|x_{\abar}^{-t},y_{-t}))\right)^2 dx^{-t}dy_{-t}\\
        \leq & B^2 \int p(x^{-t},y_{-t}) \\&\left(2\cdot \sqrt{\frac{1}{2}D_{KL}(p(y_t|x^{-t},y_{-t})||p(y_t|x_{\abar}^{-t},y_{-t}))}\right)^2 dx^{-t}dy_{-t}\\
        = & 2B^2 \int p(x^{-t},y_{-t}) D_{KL}(p(y_t|x^{-t},y_{-t})||p(y_t|x_{\abar}^{-t},y_{-t})) dx^{-t}dy_{-t}\\
        = & 2B^2 \E_{X,Y_{-t}}[D_{KL}(p(y_t|x^{-t},y_{-t})||p(y_t|x_{\abar}^{-t},y_{-t}))]\\
        = & 2B^2I(Y_t,X_A^{-t}|X_{\abar}^{-t},Y_{-t}) = 2B^2\cdot TE_{X_{A}\rightarrow Y| X_{\abar}}.
    \end{align*}
    }
The first equality rewrites the inner expected values as integrals. Then, the inequality holds applying the absolute value in the inner term and recalling that $|Y_t|\leq B$ by assumption. The second equality follows since the total variation distance is half the $L_1$ norm. Finally, the second inequality follows from Pinsker's inequality and the last three equations simply rearrange the terms and apply the definitions of conditional mutual information and transfer entropy.

Combining the three terms analysed, the result of the theorem follows.
\end{proof}

\begin{proof}[Proof of Theorem \ref{thm:classErr}]
    The optimal class that can be predicted, given the inputs, is the one with the largest probability. Therefore, it is equal to {\fontsize{7.5}{8}$y_t^* = \arg\max_{y_t\in\mathcal{Y}}p(y_t|X^{-t},Y_{-t})$} or {\fontsize{7.5}{8}$y^*_{t,\abar} = \arg\max_{y_t\in\mathcal{Y}}p(y_t|X_{\abar}^{-t},Y_{-t})$} respectively if the full set of features or only the selected ones is considered. 

    The expected prediction error considering the selected features, for the best model, is therefore equal to:
    {\fontsize{7.5}{8}
    \begin{equation*}
    \begin{aligned}
    &\inf_{g\in \mathcal{G}_{\abar}}\E_{X,Y}\left[\mathbb{1}_{\{Y_t\neq g(X_{\abar}^{-t},Y_{-t})\}}\right] = \E_{X,Y}\left[ \mathbb{1}_{\{Y_t\neq Y^*_{t,\abar}\}} \right] \\&= \epsilon + \E_{X,Y}\left[ \mathbb{1}_{\{Y_t\neq Y^*_{t,\abar}\}}-\mathbb{1}_{\{Y_t\neq Y^*_{t}\}} \right],
    \end{aligned}
\end{equation*}
}
where the second equality follows summing and subtracting the term $\epsilon=\E_{X,Y}\left[ \mathbb{1}_{\{Y_t\neq Y^*_{t}\}}\right]$, that is equal to the irreducible Bayes error suffered considering the entire set of features. 

We now focus on the expected value obtained, rewriting it as an integral:
{\fontsize{7.5}{8}
\begin{align*}
    & \E_{X,Y} \left[ \mathbb{1}_{\{Y_t\neq Y^*_{t,\abar}\}}-\mathbb{1}_{\{Y_t\neq Y^*_{t}\}} \right] \\ 
    & = \int p(x^{-t},y_{-t}) ( \E_{y_t}[ \mathbb{1}_{\{Y_t\neq Y^*_{t,\abar}\}}  |X^{-t},Y_{-t}] \\& - \E_{y_t}[ \mathbb{1}_{\{Y_t\neq Y^*_{t}\}} |X^{-t},Y_{-t}]) dx^{-t}dy_{-t}.
\end{align*}
}
Focusing on the inner expected values:
{\fontsize{7.5}{8}
\begin{align*}
    &\E_{y_t}\left[ \mathbb{1}_{\{Y_t\neq Y^*_{t}\}}|X^{-t},Y_{-t} \right] = 1-\E_{y_t}\left[ \mathbb{1}_{\{Y_t= Y^*_{t}\}}|X^{-t},Y_{-t} \right] \\ & = 1-p(Y^*_{t}|X^{-t},Y_{-t}),
\end{align*}
}
where the equalities follow from the definition of the expected value of the complement of a set and of the expected value of the indicator function. Intuitively, the equality states that the expected error when the target is estimated with its maximum probability given the features and the past of the target is equal to $1$ minus the probability of correctly estimating the target with its maximum conditional expected probability given the features and the past values of the target. A similar result holds conditioning only on the selected features, allowing rewriting the equation above as:
{\fontsize{7.5}{8}
\begin{align*}
    \E_{X,Y} & \left[ \mathbb{1}_{\{Y_t\neq Y^*_{t,\abar}\}}-\mathbb{1}_{\{Y_t\neq Y^*_{t}\}} \right] \\ 
    & = \int p(x^{-t},y_{-t}) (p(y^*_{t}|x^{-t},y_{-t}) \\
    & -p(y^*_{t,\abar}|x^{-t},y_{-t})) dx^{-t}dy_{-t}\\
    & = \int p(x^{-t},y_{-t}) (p(y^*_{t}|x^{-t},y_{-t}) -p(y^*_{t,\abar}|x^{-t},y_{-t}) \\& \pm p(y^*_{t,\abar}|x_{\abar}^{-t},y_{-t}) ) dx^{-t}dy_{-t}.\\
\end{align*}
}

Focusing on each inner term individually:
{\fontsize{7.5}{8}
\begin{align*}
    &p(y^*_{t}|x^{-t},y_{-t}) -p(y^*_{t,\abar}|x_{\abar}^{-t},y_{-t}) \\& = \max_{y_t\in \mathcal{Y}} p(y_t|x^{-t},y_{-t})-\max_{y_t\in \mathcal{Y}} p(y_{t}|x_{\abar}^{-t},y_{-t})\\
    & \leq \max_{y_t\in \mathcal{Y}} | p(y_t|x^{-t},y_{-t})-p(y_{t}|x_{\abar}^{-t},y_{-t}) | \\
    &\leq D_{TV}(p(y_t|x^{-t},y_{-t})||p(y_{t}|x_{\abar}^{-t},y_{-t}))\\
    & \leq \sqrt{\frac{1}{2}D_{KL}(p(y_t|x^{-t},y_{-t})||p(y_{t}|x_{\abar}^{-t},y_{-t}))},
\end{align*}
}
where the first equality follows by definition of $y_t^*,y^*_{t,\abar}$, the first inequality holds since the difference of maxima is smaller than or equal to the maximum absolute value of the difference and the second inequality holds by definition of total variation distance. Finally, applying Pinsker's inequality, the third inequality holds. Similarly, the second term inside the integral can be bounded as follows:
{\fontsize{7.5}{8}
\begin{align*}
    &p(y^*_{t,\abar}|x_{\abar}^{-t},y_{-t})-p(y^*_{t,\abar}|x^{-t},y_{-t}) \\
    &\leq \sqrt{\frac{1}{2}D_{KL}(p(y_t|x^{-t},y_{-t})||p(y_{t}|x_{\abar}^{-t},y_{-t}))}. 
\end{align*}
}
Therefore, the integral becomes:
{\fontsize{7.5}{8}
\begin{align*}
    &\int p(x^{-t},y_{-t})  (p(y^*_{t}|x^{-t},y_{-t}) -p(y^*_{t,\abar}|x^{-t},y_{-t}) \\&\pm p(y^*_{t,\abar}|x_{\abar}^{-t},y_{-t}) ) dx^{-t}dy_{-t}\\
    & \leq \int p(x^{-t},y_{-t}) \sqrt{2D_{KL}(p(y_t|x^{-t},y_{-t})||p(y_{t}|x_{\abar}^{-t},y_{-t}))} dx^{-t}dy_{-t}\\
    & \leq \sqrt{2\int p(x^{-t},y_{-t}) D_{KL}(p(y_t|x^{-t},y_{-t})||p(y_{t}|x_{\abar}^{-t},y_{-t}))} dx^{-t}dy_{-t}\\
    & = \sqrt{2\E_{X,Y_{-t}} [D_{KL}(p(y_t|x^{-t},y_{-t})||p(y_{t}|x_{\abar}^{-t},y_{-t}))] } \\
    & = \sqrt{2I(Y_t;X_A^{-t}|X_{\abar}^{-t},Y_{-t})} = \sqrt{2\cdot TE_{X_{A}\rightarrow Y| X_{\abar}}}.
\end{align*}
}
Substituting this result in the first equation of this proof we obtain Equation \ref{eq:classErr} of the main paper, proving the theorem.
\end{proof}

\begin{proof}[Proof of Lemma \ref{thm:recursiveBackward}]
    Let {\fontsize{7.5}{8}$X_{\hat{i}}=\arg\min_{X_i\in X_{\abar}} TE_{X_{i}\rightarrow Y|X_{\abar}\setminus X_{i}}$}, then $T(X_A\cup X_{\hat{i}})$ is defined as:
    {\fontsize{7.5}{8}
    \begin{align*}
    &T(X_{A}\cup X_{\hat{i}}) = TE_{X_{A}\cup X_{\hat{i}}\rightarrow Y|X_{\abar}\setminus X_{\hat{i}}}\\&=I(Y_t;\{X_{A}\cup X_{\hat{i}}\}^{\xtpast}|\{X_{\abar}\setminus X_{\hat{i}}\}^{\xtpast},Y_{\ytpast}).
\end{align*}
}
Then, applying multiple times the chain rule of the conditional mutual information and the definition of {\fontsize{7.5}{8}$T(X_A)=TE_{X_{A}\rightarrow Y|X_{\abar}}$} and {\fontsize{7.5}{8}$T_{\abar}=\min_{X_i\in X_{\abar}} TE_{X_{i}\rightarrow Y|X_{\abar}\setminus X_{i}}$}, the following equalities hold, proving the theorem:
{\fontsize{7.5}{8}
\begin{align*}
    &T(X_{A}\cup X_{\hat{i}}) = TE_{X_{A}\cup X_{\hat{i}}\rightarrow Y|X_{\abar}\setminus X_{\hat{i}}}\\
    & = I(Y_t;X^{\xtpast},Y_{\ytpast}) - I(Y_t;\{{X_{\abar}\setminus X_{\hat{i}}}\}^{\xtpast},Y_{\ytpast})\\
    & = I(Y_t;X_{\abar}^{\xtpast},Y_{\ytpast}) \\&+ I(Y_t;X_{A}^{\xtpast}|X_{\abar}^{\xtpast},Y_{\ytpast})\\
    & \quad -I(Y_t;\{{X_{\abar}\setminus X_{\hat{i}}}\}^{\xtpast},Y_{\ytpast})\\
    & = I(Y_t;\{X_{\abar}\setminus X_{\hat{i}}\}^{\xtpast},X_{\hat{i}}^{\xtpast},Y_{\ytpast}) + T(X_A)\\
    & \quad - I(Y_t;\{{X_{\abar}\setminus X_{\hat{i}}}\}^{\xtpast},Y_{\ytpast})\\
    & = T(X_A) + I(Y_t;X_{\hat{i}}^{\xtpast}|\{X_{\abar}\setminus X_{\hat{i}}\}^{\xtpast},Y_{\ytpast})\\
    & \quad \pm I(Y_t;\{{X_{\abar}\setminus X_{\hat{i}}}\}^{\xtpast},Y_{\ytpast})\\
    & = T(X_A)\cup T_{\abar}.
\end{align*}
}
\end{proof}

\begin{proof}[Proof of Lemma \ref{thm:TEaddingAFeature}]
Let us consider $TE_{X_{A}\cup X_i \rightarrow Y}$, which is defined in terms of conditional mutual information as:
{\fontsize{7.5}{8}
$$ TE_{X_{A}\cup X_i \rightarrow Y} = I(Y_t;\{X_{A}\cup X_i\}^{\xtpast}|Y_{\ytpast}). $$
}
Therefore, applying the chain rule the result follows:
{\fontsize{7.5}{8}
\begin{align*}
         TE_{X_{A}\cup X_i \rightarrow Y} & = I(Y_t;\{X_{A}\cup X_i\}^{\xtpast}|Y_{\ytpast})\\
        &  =   I(Y_t;X_i^{\xtpast}|X_{A}^{\xtpast},Y_{\ytpast}) \\
        &+ I(Y_t;X_{A}^{\xtpast}|Y_{\ytpast})\\
        & =
        TE_{X_{i}\rightarrow Y|X_{A}} + TE_{X_{A}\rightarrow Y}.
    \end{align*}
}
\end{proof}

\begin{proof}[Proof of Theorem \ref{thm:forwError}]
In the forward approach of Algorithm \ref{alg:TEFSalgoForw}, which selects the set of features $X_A$, the ideal regression and classification error are respectively equal to {\fontsize{7.5}{8}$\inf_{g\in \mathcal{G}_{A}}\E_{X,Y}[(Y_t-g(X_{A}^{\xtpast},Y_{\ytpast}))^2]$} and {\fontsize{7.5}{8}$\inf_{g\in \mathcal{G}_{A}}\E_{X,Y}\left[\mathbb{1}_{Y_t\neq g(X_{A}^{\xtpast},Y_{\ytpast})}\right]$}. To prove the theorem, we first need the following remark, which allows us to rewrite in a forward fashion the upper bound to the ideal error.

\begin{remark}\label{rem:errorReformulation}
    From Theorem \ref{thm:regrErr} and \ref{thm:classErr} of the main paper, applying the chain rule of conditional mutual information, the ideal regression mean squared error and classification Bayes error can be bounded as:
    {\fontsize{7.5}{8}
    \begin{align*}
    &\inf_{g\in \mathcal{G}_{A}}\E_{X,Y}[(Y_t-g(X_{A}^{\xtpast},Y_{\ytpast}))^2] \\
    &\leq \sigma^2 + 2B^2\cdot TE_{X_{\abar}\rightarrow Y| X_{A}} = \sigma^2 + 2B^2\cdot (TE_{X\rightarrow Y} - TE_{X_A \rightarrow Y}),\\
    &\inf_{g\in \mathcal{G}_{A}}\E_{X,Y}\left[\mathbb{1}_{Y_t\neq g(X_{A}^{\xtpast},Y_{\ytpast})}\right] \leq \epsilon + \sqrt{2\cdot TE_{X_{\abar}\rightarrow Y| X_{A}}}\\
    &= \epsilon + \sqrt{2\cdot (TE_{X\rightarrow Y} - TE_{X_A \rightarrow Y})}.
\end{align*}
}
This rephrasing of the statements shows that the maximisation performed in the feature selection approach of Algorithm \ref{alg:TEFSalgoForw} is consistent with the theory, since the maximisation of the term $TE_{X_A \rightarrow Y}$ minimizes the error.
\end{remark}

We are now ready to prove the relationship between the regression and classification error and the chosen amount of information desired, $\Delta$.

Recalling that at the first iteration of Algorithm \ref{alg:TEFSalgoForw} we initialize the set of selected features to be empty, $X_A=\emptyset$, it follows that at the beginning $TE_{X_A\rightarrow Y}=0$. Recursively applying Equation \ref{eq:TEaddingAFeature}, after the final iteration we have $TE_{X_A\rightarrow Y} = \sum_{k=1}^K TE_{X_k\rightarrow Y | X_{A_k} },$ where $K$ is the number of iterations of the while cycle and $X_k,X_{A_k}$ are respectively the selected feature and the set of selected features during the k-th iteration. Therefore, when the algorithm stops we have $TE_{X_A\rightarrow Y} \geq \frac{\Delta}{2B^2}$ (or $\frac{\Delta^2}{2}$ in classification). Therefore, the theorem follows by substituting these values in the reformulation of the ideal regression and classification error introduced in Remark \ref{rem:errorReformulation}
\end{proof}

\begin{proof}[Proof of Theorem \ref{corollary:finiteSample}]
In the backward case, the following inequality holds from Theorem \ref{thm:backError}:
{\fontsize{7.5}{8}
    \begin{align*}
        &\inf_{g\in \mathcal{G}_{\abar}}\E_{X,Y}[(Y_t-g(X_{\abar}^{\xtpast},Y_{\ytpast}))^2] \\&\leq \sigma^2+2B^2\cdot \sum_{k=1}^K TE_{X_{k}\rightarrow Y|X_{\abar_k}\setminus X_{k}},
    \end{align*}
}
    where $K$ is the number of iterations of the while cycle and $X_k,X_{\abar_k}$ are respectively the removed feature and the set of candidate selected features during the $k-th$ iteration. Substituting each transfer entropy term with the confidence interval of its estimator from Equation \ref{eq:CMIconcentration}, and recalling that at the end of the algorithm {\fontsize{7.5}{8}$\sum_{k=1}^K \hat{TE}_{X_{k}\rightarrow Y|X_{\abar_k}\setminus X_{k}}\leq \frac{\delta}{2B^2}$}, the result follows. Notice that, to guarantee that any possible subset of features contemporaneously satisfies the concentration inequality of Equation \ref{eq:CMIconcentration} with probability $\eta$, a term $2^d$ is added in the concentration exploiting the union bound over all possible subsets of features.

    Similarly, in the forward case, the following inequality holds from Theorem \ref{thm:forwError}:
    {\fontsize{7.5}{8}
    \begin{equation*}
    \inf_{g\in \mathcal{G}_{A}}\E_{X,Y}[(Y_t-g(X_{A}^{\xtpast},Y_{\ytpast}))^2] \leq 
    \sum_{k=1}^K TE_{X_k\rightarrow Y | X_{A_k} },
    \end{equation*}
    }
    where $K$ is the number of iterations of the while cycle and $X_k,X_{A_k}$ are respectively the selected feature and the set of selected features during the k-th iteration. Again, substituting each transfer entropy term with the high-probability confidence interval of its estimator from Equation \ref{eq:CMIconcentration}, and recalling that at the end of the algorithm the cumulative selected information {\fontsize{7.5}{8}$\sum_{k=1}^K \hat{TE}_{X_k\rightarrow Y | X_{A_k} }\geq \frac{\delta}{2B^2}$}, the result follows.
\end{proof}

\section{Experimental analysis: additional details and results}\label{app:experiments}

In this section, the experiments performed both with synthetic data and with real-world datasets are discussed in detail. We ran the experiments on a server with 88 Intel(R) Xeon(R) CPU E7-8880 v4 @ 2.20GHz cpus and 94 GB of RAM. We used 8 cpus of the server and the overall computational time was less than 24 hours. As reported in the main paper, code and datasets can be found at \url{https://www.dropbox.com/scl/fo/ctvxqh8q131f8sim4s22o/h?dl=0&rlkey=5hsbpold82bi8vilu44ubb508} and they will be available on GitHub after the double-blind review process.

\subsection{Synthetic Experiments}

In the main paper, Figure \ref{fig:linearSynth} shows the three causal graphs considered for the three, five and ten-dimensional synthetic experiments, together with the time lag and the coefficient selected to build the benchmark dataset. The graphs are designed to consider some classical challenges of time-series causal discovery: autocorrelation, time lags, noise, and spurious association due to common causal confounders. As a first experiment, in all three cases, $N=300$ samples have been drawn (corresponding to considering time series of length equal to $300$). In particular, each node without causal parents has been associated with a feature of the real-world experiment on droughts that considers fifteen features discussed in the next subsection ($X_0$ for the three-dimensional dataset, $X_1,X_3$ for the five-dimensional dataset, $X_0,X_5,X_7$ for the ten-dimensional one). In this way, a realistic time series is considered for each of these nodes. Then, following the causal dependencies, the other features and the target have been designed, considering a linear contribution of each of their (lagged) parent variables through the coefficient reported in Figure \ref{fig:linearSynth}. Moreover, Gaussian noise with variance $\sigma=0.1$ has been added to each feature. 
Considering a varying seed from $0$ to $9$, the algorithms have been run ten times to produce the results in terms of true positive rate (TPR) and false positive rate (TPR). In particular, forward and backward CMI feature selection have been considered as baselines, applying the algorithms provided in~\citep{Beraha2019}. Moreover, the PCMCI algorithm~\citep{runge2019detecting} has been considered as state-of-the-art causal discovery method, applying the implementation provided by the authors of the method in the Pyhton library Tigramite (\url{https://github.com/jakobrunge/tigramite.git}). Given the linear nature of the problem and the much faster computational time, PCMCI algorithm has been run considering linear independence tests based on partial correlations, leaving the other parameters with default values. Moreover, to detect the majority of causal links, the proposed algorithms \emph{Forward and Backward TEFS}, have been run considering respectively a small value of the maximum information loss $\delta=10^{-6}$ and a large value of the required minimum information gain $\Delta=100$. 

To further analyse the impact of the coefficients, the number of samples, the noise and the time lags, different values of each of the parameters have been considered, keeping the others fixed as discussed above and repeating the experiments for each configuration of value ten times, with seeds from $0$ to $9$. In particular, the coefficients have been randomly sampled five additional times from a uniform distribution in the interval $[-1,-0.5]\cup [0.5,1]$ (not considering the interval around zero $[-0.5,0.5]$ to avoid too weak causal links). The variance of the Gaussian noise has been increased five times, considering the interval $[0.1,0.3,0.5,0.7,0.9]$ for the three-dimensional setting and $[0.01,0.05,0.1,0.15,0.2]$ for the five and ten-dimensional ones (to avoid too noisy experiments). Then, the number of samples has been varied, considering five increasing values $[100,200,300,400,500]$. Finally, the maximum time lag that the algorithms should consider has been tested, considering the exact value $2$ together with $3,4,5$. As discussed in the main paper, in most cases the proposed methods are able to correctly detect the correct causal links, with TPR better or equivalent to the PCMCI algorithm, controlling the FPR better. Moreover, the proposed methods are sufficiently robust \wrt increase of noise, a decrease of the number of samples and with different causal links (identified by different random coefficients). Additionally, they have a satisfactory performance also considering different time depths, with a worsening of the TPR only with large depth. 

In the main paper, Table \ref{tab:synth3D} shows the TPR and FPR for all the considered algorithms for the three-dimensional setting. In this section, Table \ref{tab:synth5D} and \ref{tab:synth10D} respectively show the same results for the five and ten-dimensional synthetic experiments. The TPR and FPR have been obtained considering the true causal links between the features and the target, for all ten repetitions of each experiment with different seeds. As an example, in the three-dimensional experiment, there are two correct and one false causal link between the features and the target (also considering its autoregressive component). Therefore, if the links are correctly detected during all ten repetitions, the TPR is equal to $\frac{20}{20}=1$. On the contrary, if an algorithm detects one incorrect causal link out of the repetitions, the FPR is equal to $\frac{1}{10}=0.1$. Finally, we consider the PCMCI to detect a link between a feature and the target if the algorithm detects one link between them, independently from the time lag that it selects. Indeed, PCMCI algorithm tries to identify not only the causal features but also the time lags that have causal effects on the actual target (see Remark \ref{rem:variant} in the main paper for a broader discussion on the choice to consider all the contribution of a feature as a whole for the forward and backward TEFS). 

\begin{table*}[t]
\caption{Five-dimensional experiments on synthetic datasets. Each experiment has been repeated ten times, with different seeds, varying one parameter at a time.
\label{tab:synth5D}}
\centering 
\centering
\resizebox{\textwidth}{!}
{\begin{tabular}{@{}cccccc@{}} \hline 
\ & Forward CMI & Backward CMI & PCMCI & \cellcolor{lightgray!25} Forward TEFS (\textbf{ours}) & \cellcolor{lightgray!25} Backward TEFS (\textbf{ours}) \\\hline 
\textbf{Benchmark} & & & & \cellcolor{lightgray!25} & \cellcolor{lightgray!25} \\
TPR & $0.47$ & $0.47$ & $0.67$ & \cellcolor{lightgray!25} $0.93$ & \cellcolor{lightgray!25} $0.93$\\
FPR & $0.50$ & $0.50$ & $0.40$ & \cellcolor{lightgray!25} $0$ & \cellcolor{lightgray!25} $0$ \\
\hline
\textbf{Increasing Noise} & & & & \cellcolor{lightgray!25} & \cellcolor{lightgray!25} \\
TPR & $[0.67,0.60,0.47,0.47,0.40]$ & $[0.67,0.60,0.47,0.47,0.40]$  & $[0.37, 0.43, 0.67, 0.80, 0.83]$ & \cellcolor{lightgray!25} $[1,1,0.93,0.87,0.77]$ & \cellcolor{lightgray!25} $[1,1,0.93,0.87,0.77]$\\
FPR & $[0.50,0.50,0.50,0.50,0.50]$ & $[0.50,0.50,0.50,0.50,0.50]$ & $[0.05, 0.05, 0.40, 0.65, 0.50]$ & \cellcolor{lightgray!25} $[0, 0, 0, 0, 0]$ & \cellcolor{lightgray!25} $[0, 0, 0, 0, 0]$ \\
\hline
\textbf{Increasing $\tau$} & & & & \cellcolor{lightgray!25} & \cellcolor{lightgray!25} \\
TPR & $[0.67,0.33,0.30,0.27]$ & $[0.67,0.33,0.30,0.27]$  & $[0.37,0.03,0.63,0.73]$ & \cellcolor{lightgray!25} $[1,0.67,0.53,0.47]$ & \cellcolor{lightgray!25} $[1,0.67,0.67,0.67]$\\
FPR & $[0.50,0.50,0.50,0.50]$ & $[0.50,0.50,0.50,0.50]$ & $[0.05,0.05,0.60,0.60]$ & \cellcolor{lightgray!25} $[0,0.5,0,0]$ & \cellcolor{lightgray!25} $[0,0.5,0,0]$ \\
\hline
\textbf{Increasing N} & & & & \cellcolor{lightgray!25} & \cellcolor{lightgray!25} \\
TPR & $[0.43, 0.57, 0.47, 0.63, 0.60]$ & $[0.43, 0.57, 0.47, 0.63, 0.60]$ & $[0.30,0.50,0.67,0.77,0.57]$ & \cellcolor{lightgray!25} $[0.70,0.90,0.93,1,0.97]$ & \cellcolor{lightgray!25} $[0.73,0.90,0.93,1,0.97]$\\
FPR & $[0.50,0,0.50,0.50,0.50]$ & $[0.50,0,0.50,0.50,0.50]$ & $[0,0.5,0.4,0.4,0.4]$ & \cellcolor{lightgray!25} $[0.2,0.05,0,0,0]$ & \cellcolor{lightgray!25} $[0.15,0.05,0,0,0]$ \\
\hline
\textbf{Random Coefficients} & & & & \cellcolor{lightgray!25} & \cellcolor{lightgray!25} \\
TPR & $[0.43,0.67,0.67,0.43,0.67]$ & $[0.43,0.67,0.67,0.43,0.67]$ & $[0.63,0.77,0.77,0.90,0.77]$ & \cellcolor{lightgray!25} $[0.90,1,1,0.67,1]$ & \cellcolor{lightgray!25} $[0.90,1,1,0.67,1]$\\
FPR & $[0.50,0,0.50,0.50,0.55]$ & $[0.50,0,0.50,0.50,0.50]$ & $[0.15,0.30,0.45,0.30,0.45]$ & \cellcolor{lightgray!25} $[0,0.05,0,0,0]$ & \cellcolor{lightgray!25} $[0,0,0,0,0]$ \\
\hline
\end{tabular}}
\end{table*}

\begin{table*}[t]
\caption{Ten-dimensional experiments on synthetic datasets. Each experiment has been repeated ten times, with different seeds, varying one parameter at a time.
\label{tab:synth10D}}
\centering
\centering
\resizebox{\textwidth}{!}
{\begin{tabular}{@{}cccccc@{}} \hline 
\ & Forward CMI & Backward CMI & PCMCI & \cellcolor{lightgray!25} Forward TEFS (\textbf{ours}) & \cellcolor{lightgray!25} Backward TEFS (\textbf{ours})\\\hline 
\textbf{Benchmark} & & & & \cellcolor{lightgray!25} & \cellcolor{lightgray!25} \\
TPR & $0.67$ & $0.67$ & $0.97$ & \cellcolor{lightgray!25} $1.0$ & \cellcolor{lightgray!25} $1.0$\\
FPR & $0.01$ & $0.01$ & $0.66$ & \cellcolor{lightgray!25} $0$ & \cellcolor{lightgray!25} $0$ \\
\hline
\textbf{Increasing Noise} & & & & \cellcolor{lightgray!25} & \cellcolor{lightgray!25} \\
TPR & $[0.67, 0.67, 0.67, 0.67, 0.67]$ & $[0.67, 0.67, 0.67, 0.67, 0.67]$ & $[1, 0.97, 0.97, 0.77, 0.87]$ & \cellcolor{lightgray!25} $[1, 1, 1, 1, 1]$ & \cellcolor{lightgray!25} $[1, 1, 1, 1, 1]$ \\
FPR & $[0,0,0.01, 0.09, 0.14]$ & $[0,0,0.01,0.09,0.13]$ & $[0.71, 0.57, 0.66, 0.57, 0.6]$ & \cellcolor{lightgray!25} $[0, 0, 0, 0, 0]$ & \cellcolor{lightgray!25} $[0, 0, 0, 0, 0]$  \\
\hline
\textbf{Increasing $\tau$} & & & & \cellcolor{lightgray!25} & \cellcolor{lightgray!25} \\
TPR & $[0.67, 0.67, 0.67,0.43]$ & $[0.67, 0.67, 0.67,0.43]$ & $[0.53, 0.97, 0.97, 1]$ & \cellcolor{lightgray!25} $[1, 1, 1, 0.7]$ & \cellcolor{lightgray!25} $[1, 1, 1, 0.7]$ \\
FPR & $[0.03, 0.01, 0, 0]$ & $[0.06, 0.01, 0, 0]$ & $[0.44, 0.66, 0.73, 0.83]$ & \cellcolor{lightgray!25} $[0.04, 0, 0, 0.01]$ & \cellcolor{lightgray!25} $[0.04, 0, 0, 0.01]$  \\
\hline
\textbf{Increasing N} & & & & \cellcolor{lightgray!25} & \cellcolor{lightgray!25} \\
TPR & $[0.67,0.67,0.67,0.67,0.67]$ & $[0.67,0.67,0.67,0.67,0.67]$ & $[0.67, 0.9, 0.97, 0.9, 0.9]$ & \cellcolor{lightgray!25} $[0.6, 0.97, 1,1,1]$ & \cellcolor{lightgray!25} $[0.93, 1, 1,1,1]$ \\
FPR & $[0.1,0.04,0.01,0.01,0]$ & $[0.2, 0.04, 0.01, 0.01, 0]$ & $[0.49,0.37,0.66,0.5,0.56]$ & \cellcolor{lightgray!25} $[0.03, 0.01, 0, 0, 0]$ & \cellcolor{lightgray!25} $[0.11, 0, 0, 0, 0]$  \\
\hline
\textbf{Random Coefficients} & & & & \cellcolor{lightgray!25} & \cellcolor{lightgray!25} \\
TPR & $[0.67,0.63,0.67,0.36,0.67]$ & $[0.67,0.63,0.67,0.36,0.67]$ & $[0.97,0.77,0.87,0.97,0.93]$ & \cellcolor{lightgray!25} $[1,0.47,0,0.67,1]$ & \cellcolor{lightgray!25} $[1,0.6,0.97,0.67,1]$ \\
FPR & $[0.01,0.14,0.23,0.19,0]$ & $[0.01,0.14,0.16,0.16,0]$ & $[0.66,0.36,0.54,0.16,0.63]$ & \cellcolor{lightgray!25} $[0,0,0,0,0]$ & \cellcolor{lightgray!25} $[0,0.03,0,0,0]$ \\
\hline
\end{tabular}}
\end{table*}

To inspect also the capability of the proposed methods to identify causal relationships with a large number of features, we also repeated the synthetic analysis with $D\in\{15,20,40,60,80,100\}$, preserving the relationships and coefficients of the first nine features and the target as in Figure \ref{fig:graph10} and adding independent variables that can only be autocorrelated or interact with each other. In particular, we iteratively add features in groups of three, where the first one is randomly sampled from a standard normal distribution, and the second one linearly depends (with coefficient randomly sampled from a uniform distribution in $[-1,1]$) on the value of the first of one timestep before plus Gaussian noise, and the third depends on the second of two timesteps before and on its value of three timesteps before (again, with random coefficients sampled from a uniform distribution in $[-1,1]$) plus Gaussian noise. This group of three variables increases the number of features without impacting the causal relationships between the features and the target. In this context, the experiments have been repeated ten times with fixed benchmark parameters as described for the ten-dimensional experiment, varying ten different seeds. The associated TPR and FPR are reported in Table \ref{tab:synthND} and show the capability of the proposed methods to control the FPR \wrt the PCMCI method, which is less accurate in discarding features, selecting more than ten features in high dimensions. Moreover, from the result, it is also possible to notice that the forward version of the TEFS algorithm is more stable in high-dimensional contexts, due to the huge number of features in the conditioning phase of the first iterations of the backward approach, which makes more difficult the estimate of the contribution of each feature on the target.

\begin{table*}[t]
\caption{Larger dimensional experiments on synthetic datasets. Each experiment has been repeated ten times, with different seeds.
\label{tab:synthND}}
\centering
\centering
\begin{tabular}{@{}cccccc@{}} \hline 
\ & Forward CMI & Backward CMI & PCMCI & \cellcolor{lightgray!25} Forward TEFS (\textbf{ours}) & \cellcolor{lightgray!25} Backward TEFS (\textbf{ours})\\\hline 
\textbf{$D=15$} & & & & \cellcolor{lightgray!25} & \cellcolor{lightgray!25} \\
TPR & $0.67$ & $0.67$ & $0.73$ & \cellcolor{lightgray!25} $1.0$ & \cellcolor{lightgray!25} $1.0$\\
FPR & $0.01$ & $0.01$ & $0.52$ & \cellcolor{lightgray!25} $0$ & \cellcolor{lightgray!25} $0$ \\
\textbf{$D=20$} & & & & \cellcolor{lightgray!25} & \cellcolor{lightgray!25} \\
TPR & $0.67$ & $0.67$ & $0.73$ & \cellcolor{lightgray!25} $1.0$ & \cellcolor{lightgray!25} $1.0$\\
FPR & $0.01$ & $0.01$ & $0.29$ & \cellcolor{lightgray!25} $0$ & \cellcolor{lightgray!25} $0$ \\
\textbf{$D=40$} & & & & \cellcolor{lightgray!25} & \cellcolor{lightgray!25} \\
TPR & $0.67$ & $0.67$ & $0.73$ & \cellcolor{lightgray!25} $1.0$ & \cellcolor{lightgray!25} $1.0$\\
FPR & $0.01$ & $0.01$ & $0.25$ & \cellcolor{lightgray!25} $0$ & \cellcolor{lightgray!25} $0$ \\
\hline
\textbf{$D=60$} & & & & \cellcolor{lightgray!25} & \cellcolor{lightgray!25} \\
TPR & $0.67$ & $0.67$ & $0.73$ & \cellcolor{lightgray!25} $1$ & \cellcolor{lightgray!25} $0.93$\\
FPR & $0.01$ & $0.01$ & $0.22$ & \cellcolor{lightgray!25} $0$ & \cellcolor{lightgray!25} $0.01$ \\
\hline 
\textbf{$D=80$} & & & & \cellcolor{lightgray!25} & \cellcolor{lightgray!25} \\
TPR & $0.67$ & $0.57$ & $0.73$ & \cellcolor{lightgray!25} $1.0$ & \cellcolor{lightgray!25} $0.90$\\
FPR & $0.01$ & $0.20$ & $0.20$ & \cellcolor{lightgray!25} $0$ & \cellcolor{lightgray!25} $0.18$ \\
\hline 
\textbf{$D=100$} & & & & \cellcolor{lightgray!25} & \cellcolor{lightgray!25} \\
TPR & $0.67$ & $0.53$ & $0.73$ & \cellcolor{lightgray!25} $1.0$ & \cellcolor{lightgray!25} $0.90$\\
FPR & $0.01$ & $0.28$ & $0.18$ & \cellcolor{lightgray!25} $0$ & \cellcolor{lightgray!25} $0.29$ \\
\hline
\end{tabular}
\end{table*}

\subsection{Real-World Experiments}
This subsection provides additional details on the real-world datasets and experiments performed in order to analyse the behaviour of the proposed algorithms in real-world scenarios. In particular, three datasets with five features have been first considered. These climate datasets are composed by climatological features and a scalar target. The targets represent the state of vegetation of three sub-basins of the Po River (Vegetation Health Index)\footnote{They have been extracted from "Peter Zellner. Vegetation health index - 231 m 8 days (version 1.0) [data set]. Eurac Research, 2022".}. The features are observational data of temperature and precipitation\footnote{Features have been retrieved from  "Richard C. Cornes, Gerard van der Schrier, Else J. M. van den Besselaar, and Philip D. Jones. An ensemble version of the e-obs temperature and precipitation data sets. Journal of Geophysical Research: Atmospheres, 123(17):9391–9409, 2018.", "K. Didan. Myd13q1 modis/aqua vegetation indices 16-day l3 global 250m sin grid v006 [data set]. NASA EOSDIS Land Processes DAAC., 2015." and averaged over the basin they refer to.}. Moreover, to consider a larger dataset, a similar dataset has been considered, with the same index as the target and fifteen features retrieved from the same sources and representing the precipitation and temperature over the sub-basin considered and its two neighbouring sub-basins. These four datasets have been designed by the authors and they are available in the repository of this work. It is expected that precipitation and temperature have an impact on the state of the vegetation, with some timesteps of lag. The timestep considered for each of the four datasets is a week and the four datasets contain $N=639$ train samples (from 2001 to 2014) and $N=228$ test samples (from 2015 to 2019).
In addition, to further analyse the behaviour of the proposed algorithms on a benchmark dataset, we consider a dataset composed of five variables available in an online causal repository (\url{https://causeme.uv.es/}), where we consider in turn each variable to be the target and the other four variables the candidate causal features (for a total of five experiments). In particular, in the repository mentioned above, we consider the dataset \emph{TESTWEATHNOISE} with $N=1000$ samples ($600$ samples for training and the following $400$ to test) and five variables, that contains classical causal discovery challenges (autocorrelation, time delays, non-linearity, observational noise). 

Together with the two proposed methods (\emph{Forward and Backward TEFS}), as done for the synthetic experiments, we consider the forward and backward CMI feature selection described in~\citep{Beraha2019}, together with the PCMCI algorithm~\citep{runge2019detecting}, considering both linear and non-linear (CMI-based) independence tests. Additionally, to also consider a baseline method related to classical Granger causality, we also consider the FullCI method implemented in the same library (\url{https://github.com/jakobrunge/tigramite.git}), considering again both the linear and non-linear versions. Since there is no ground truth on the causal relationships between the features and the target, we consider the $R^2$ test score performing linear regression as an index of the performance of the algorithms (considering as inputs also the last $M$ values of the target for our approaches and when selected by the PCMCI and Full\_CI algorithms). This is motivated by the nature of the proposed methodology, which is designed to select features through a causal quantity but finally aimed to select the most relevant features in terms of regression performance. Since the conditioning on the last values of the target itself is the basic concept to measure causality with transfer entropy, we also provide the $R^2$ test score considering only the autoregressive component. Moreover, we consider $L=M\in\{1,2,3\}$ as past timesteps considered by the algorithms, since in real-world problems there is no knowledge on the number of past values that have an impact on the current value of the target.

\begin{table*}[ht]
\caption{Real-world experiments for three climate datasets with five features and one climate dataset with fifteen features. Each experiment has been repeated three times, considering different time-depths $M=L\in\{1,2,3\}$. The table reports the set of IDs of the selected features and the corresponding $R^2$ test score.
\label{tab:ClimRealFull}}
\centering 
\centering
\begin{tabular}{@{}cccc@{}} 
\hline
& \multicolumn{3} {c} {\textbf{Climate, 5 features, first dataset}}\\
\ & $M=L=1$ & $M=L=2$ & $M=L=3$ \\
\hline 
Autoregression & $\{\}\rightarrow0.28$ & $\{\}\rightarrow0.28$ & $\{\}\rightarrow0.27$ \\
Forward CMI & $\{0,3,4\}\rightarrow0.34$ & $\{0,1\}\rightarrow0.36$ & $\{0,1\}\rightarrow0.35$ \\
Backward CMI & $\{0,1,2,3,4\}\rightarrow0.34$ & $\{0,3,4\}\rightarrow0.26$ & $\{0,1,3\}\rightarrow0.35$ \\
FullCI\_linear & $\{0,2,4\}\rightarrow0.45$ & $\{0,1,3,4\}\rightarrow0.47$  & $\{0,3,4\}\rightarrow0.48$ \\
FullCI\_nonLinear & $\{0,4\}\rightarrow0.44$ & $\{0,4\}\rightarrow0.42$  & $\{0,3,4\}\rightarrow0.48$ \\
PCMCI\_linear & $\{0,2,4\}\rightarrow0.45$ & $\{0,2,4\}\rightarrow0.43$  & $\{0,2,3,4\}\rightarrow0.47$ \\
PCMCI\_nonLinear & $\{0,1,2,4\}\rightarrow0.43$ & $\{0,2,3,4\}\rightarrow0.42$  & $\{0,1,3,4\}\rightarrow0.47$ \\
\cellcolor{lightgray!25} Forward TEFS (\textbf{ours}) & \cellcolor{lightgray!25} $\{0,4\}\rightarrow0.44$   & \cellcolor{lightgray!25} $\{0,4\}\rightarrow0.42$ & \cellcolor{lightgray!25} $\{0,4\}\rightarrow0.47$ \\
\cellcolor{lightgray!25} Backward TEFS (\textbf{ours}) & \cellcolor{lightgray!25} $\{0,4\}\rightarrow0.44$   & \cellcolor{lightgray!25} $\{0,4\}\rightarrow0.42$ & \cellcolor{lightgray!25} $\{0,4\}\rightarrow0.47$ \\
\hline
& \multicolumn{3} {c} {\textbf{Climate, 5 features, second dataset}}\\
\ & $M=L=1$ & $M=L=2$ & $M=L=3$ \\
\hline 
Autoregression & $\{\}\rightarrow0.29$ & $\{\}\rightarrow0.30$ & $\{\}\rightarrow0.29$ \\
Forward CMI & $\{0,1,2,3,4\}\rightarrow0.24$ & $\{1,2,3,4\}\rightarrow0.23$ & $\{2,3\}\rightarrow0.25$ \\
Backward CMI & $\{0,1,2,3,4\}\rightarrow0.24$ & $\{1,2,3,4\}\rightarrow0.23$ & $\{2,3\}\rightarrow0.25$ \\
FullCI\_linear & $\{0,1,2\}\rightarrow0.41$ & $\{0,2,4\}\rightarrow0.45$ & $\{2,3,4\}\rightarrow0.46$ \\
FullCI\_nonLinear & $\{0,1,2\}\rightarrow0.41$ & $\{2\}\rightarrow0.42$ & $\{2\}\rightarrow0.47$ \\
PCMCI\_linear & $\{1,2,3\}\rightarrow0.43$ & $\{0,1,2,3\}\rightarrow0.47$ & $\{1,2,3\}\rightarrow0.47$ \\
PCMCI\_nonLinear & $\{1,2,3,4\}\rightarrow0.43$ & $\{1,2,3,4\}\rightarrow0.40$ & $\{1,2,3,4\}\rightarrow0.46$ \\
\cellcolor{lightgray!25} Forward TEFS (\textbf{ours}) & \cellcolor{lightgray!25} $\{2\}\rightarrow0.43$  & \cellcolor{lightgray!25} $\{2\}\rightarrow0.42$ & \cellcolor{lightgray!25} $\{2\}\rightarrow0.47$ \\
\cellcolor{lightgray!25} Backward TEFS (\textbf{ours}) & \cellcolor{lightgray!25} $\{2\}\rightarrow0.43$  & \cellcolor{lightgray!25} $\{2\}\rightarrow0.42$  & \cellcolor{lightgray!25} $\{2\}\rightarrow0.47$ \\
\hline
& \multicolumn{3} {c} {\textbf{Climate, 5 features, third dataset}}\\
\ & $M=L=1$ & $M=L=2$ & $M=L=3$ \\
\hline 
Autoregression & $\{\}\rightarrow0.29$ & $\{\}\rightarrow0.30$ & $\{\}\rightarrow0.29$ \\
Forward CMI & $\{0,1,2,3\}\rightarrow0.17$ & $\{0,1,2,3\}\rightarrow0.19$ & $\{0,1,2\}\rightarrow0.19$ \\
Backward CMI & $\{0,1,2,3\}\rightarrow0.17$ & $\{0,1,2,3\}\rightarrow0.19$ & $\{2,3,4\}\rightarrow0.15$ \\
FullCI\_linear & $\{1,2,3,4\}\rightarrow0.44$ & $\{0,1,2\}\rightarrow0.40$ & $\{0,1,2\}\rightarrow0.45$ \\
FullCI\_nonLinear & $\{1,2,3\}\rightarrow0.40$ & $\{1\}\rightarrow0.38$ & $\{1\}\rightarrow0.43$ \\
PCMCI\_linear & $\{1,2\}\rightarrow0.40$ & $\{1,2,3\}\rightarrow0.40$ & $\{1,2\}\rightarrow0.44$ \\
PCMCI\_nonLinear & $\{1,2\}\rightarrow0.40$ & $\{1,2,4\}\rightarrow0.41$ & $\{1,2\}\rightarrow0.44$ \\
\cellcolor{lightgray!25} Forward TEFS (\textbf{ours}) & \cellcolor{lightgray!25} $\{1,2\}\rightarrow0.40$ & \cellcolor{lightgray!25} $\{1\}\rightarrow0.38$ & \cellcolor{lightgray!25} $\{1\}\rightarrow0.43$ \\
\cellcolor{lightgray!25} Backward TEFS (\textbf{ours}) & \cellcolor{lightgray!25} $\{1,2\}\rightarrow0.40$ & \cellcolor{lightgray!25} $\{1\}\rightarrow0.38$ & \cellcolor{lightgray!25} $\{1\}\rightarrow0.43$ \\
\hline
& \multicolumn{3} {c} {\textbf{Climate, 15 features}}\\
\ & $M=L=1$ & $M=L=2$ & $M=L=3$ \\
\hline 
Autoregression & $\{\}\rightarrow0.28$ & $\{\}\rightarrow0.28$ & $\{\}\rightarrow0.27$ \\
Forward CMI & $\{0,3,4,11,12\}\rightarrow0.29$ & $\{0,2,11\}\rightarrow0.30$ & $\{0,5\}\rightarrow0.28$ \\
Backward CMI & $\{0,2,3,11\}\rightarrow0.29$ & $\{0,3,4\}\rightarrow0.26$ & $\{1,2,3,12\}\rightarrow0.26$ \\
FullCI\_linear & $\{2,5,7,12,14\}\rightarrow0.45$ & $\{5,7,11\}\rightarrow0.47$ & $\{12\}\rightarrow0.32$ \\
FullCI\_nonLinear & $\{12\}\rightarrow0.32$ & $\{\}\rightarrow0.28$ & $\{13\}\rightarrow0.29$ \\
PCMCI\_linear & $\{0,2,7\}\rightarrow0.45$ & $\{0,2,7,9,13\}\rightarrow0.42$ & $\{2,7\}\rightarrow0.48$ \\
PCMCI\_nonLinear & $\{2,7\}\rightarrow0.44$ & $\{0,2,7,9,10\}\rightarrow0.41$ & $\{0,7,12,13\}\rightarrow0.48$ \\
\cellcolor{lightgray!25} Forward TEFS (\textbf{ours}) & \cellcolor{lightgray!25} $\{0,7\}\rightarrow0.44$  & \cellcolor{lightgray!25} $\{0,7\}\rightarrow0.43$ & \cellcolor{lightgray!25} $\{7\}\rightarrow0.47$ \\
\cellcolor{lightgray!25} Backward TEFS (\textbf{ours}) & \cellcolor{lightgray!25} $\{0,11\}\rightarrow0.44$ & \cellcolor{lightgray!25} $\{0,11\}\rightarrow0.42$ & \cellcolor{lightgray!25} $\{11\}\rightarrow0.43$ \\
\hline
\end{tabular}
\end{table*}

\begin{table*}[ht]
\caption{Real-world experiments for benchmark dataset, with five variables, considering one variable at time as target. Each experiment has been repeated three times, considering different time-depths $M=L\in\{1,2,3\}$. The table reports the set of IDs of the selected features and the corresponding $R^2$ test score.
\label{tab:BenchRealFull}}
\centering 
\centering
\begin{tabular}{@{}cccc@{}} 
\hline
& \multicolumn{3} {c} {\textbf{Benchmark, variable 0 as target}}\\
\ & $M=L=1$ & $M=L=2$ & $M=L=3$ \\
\hline 
Autoregression & $\{\}\rightarrow0.10$ & $\{\}\rightarrow0.12$ & $\{\}\rightarrow0.12$ \\
Forward CMI & $\{2\}\rightarrow0.01$ & $\{3\}\rightarrow0.01$ & $\{1\}\rightarrow0.01$ \\
Backward CMI & $\{1,3\}\rightarrow0.01$ & $\{3\}\rightarrow0.01$ & $\{1\}\rightarrow0.01$ \\
FullCI\_linear & $\{\}\rightarrow0.10$ & $\{\}\rightarrow0.12$  & $\{\}\rightarrow0.12$ \\
FullCI\_nonLinear & $\{\}\rightarrow0.10$ & $\{\}\rightarrow0.12$  & $\{\}\rightarrow0.12$ \\
PCMCI\_linear & $\{\}\rightarrow0.10$ & $\{\}\rightarrow0.12$  & $\{\}\rightarrow0.12$ \\
PCMCI\_nonLinear & $\{\}\rightarrow0.10$ & $\{\}\rightarrow0.12$  & $\{\}\rightarrow0.12$ \\
\cellcolor{lightgray!25} Forward TEFS (\textbf{ours}) & \cellcolor{lightgray!25} $\{\}\rightarrow0.10$  & \cellcolor{lightgray!25} $\{\}\rightarrow0.12$ & \cellcolor{lightgray!25} $\{\}\rightarrow0.12$ \\
\cellcolor{lightgray!25} Backward TEFS (\textbf{ours}) & \cellcolor{lightgray!25} $\{\}\rightarrow0.10$   & \cellcolor{lightgray!25} $\{\}\rightarrow0.12$ & \cellcolor{lightgray!25} $\{\}\rightarrow0.12$ \\
\hline
& \multicolumn{3} {c} {\textbf{Benchmark, variable 1 as target}}\\
\ & $M=L=1$ & $M=L=2$ & $M=L=3$ \\
\hline 
Autoregression & $\{\}\rightarrow0$ & $\{\}\rightarrow0$ & $\{\}\rightarrow0$ \\
Forward CMI & $\{0\}\rightarrow0$ & $\{2,3\}\rightarrow0$ & $\{3\}\rightarrow0$ \\
Backward CMI & $\{2,3\}\rightarrow0$ & $\{2,3\}\rightarrow0$ & $\{3\}\rightarrow0$ \\
FullCI\_linear & $\{\}\rightarrow0$ & $\{2,4\}\rightarrow0$ & $\{2,4\}\rightarrow0$ \\
FullCI\_nonLinear & $\{\}\rightarrow0$ & $\{\}\rightarrow0$ & $\{4\}\rightarrow0$ \\
PCMCI\_linear & $\{\}\rightarrow0$ & $\{2,4\}\rightarrow0$ & $\{2,4\}\rightarrow0$ \\
PCMCI\_nonLinear & $\{\}\rightarrow0$ & $\{2\}\rightarrow0$ & $\{2\}\rightarrow0$ \\
\cellcolor{lightgray!25} Forward TEFS (\textbf{ours}) & \cellcolor{lightgray!25} $\{\}\rightarrow0$  & \cellcolor{lightgray!25} $\{\}\rightarrow0$ & \cellcolor{lightgray!25} $\{\}\rightarrow0$ \\
\cellcolor{lightgray!25} Backward TEFS (\textbf{ours}) & \cellcolor{lightgray!25} $\{\}\rightarrow0$  & \cellcolor{lightgray!25} $\{\}\rightarrow0$  & \cellcolor{lightgray!25} $\{\}\rightarrow0$ \\
\hline
& \multicolumn{3} {c} {\textbf{Benchmark, variable 2 as target}}\\
\ & $M=L=1$ & $M=L=2$ & $M=L=3$ \\
\hline 
Autoregression & $\{\}\rightarrow0.01$ & $\{\}\rightarrow0.01$ & $\{\}\rightarrow0.03$\\
Forward CMI & $\{3,4\}\rightarrow0.09$ & $\{3\}\rightarrow0.07$ & $\{0,3\}\rightarrow0.07$ \\
Backward CMI & $\{3,4\}\rightarrow0.09$ & $\{3\}\rightarrow0.07$ & $\{0,3\}\rightarrow0.07$ \\
FullCI\_linear  & $\{3,4\}\rightarrow0.09$ & $\{3,4\}\rightarrow0.13$ & $\{0,3,4\}\rightarrow0.13$ \\
FullCI\_nonLinear  & $\{3,4\}\rightarrow0.09$ & $\{3,4\}\rightarrow0.13$ & $\{3,4\}\rightarrow0.12$ \\
PCMCI\_linear  & $\{3,4\}\rightarrow0.09$ & $\{3,4\}\rightarrow0.13$ & $\{3,4\}\rightarrow0.12$ \\
PCMCI\_nonLinear  & $\{3,4\}\rightarrow0.09$ & $\{3,4\}\rightarrow0.13$ & $\{3,4\}\rightarrow0.12$ \\
\cellcolor{lightgray!25} Forward TEFS (\textbf{ours}) & \cellcolor{lightgray!25} $\{3,4\}\rightarrow0.09$ & \cellcolor{lightgray!25} $\{4\}\rightarrow0.09$ & \cellcolor{lightgray!25} $\{3\}\rightarrow0.11$ \\
\cellcolor{lightgray!25} Backward TEFS (\textbf{ours}) & \cellcolor{lightgray!25} $\{3,4\}\rightarrow0.09$ & \cellcolor{lightgray!25} $\{4\}\rightarrow0.09$ & \cellcolor{lightgray!25} $\{3\}\rightarrow0.11$ \\
\hline
& \multicolumn{3} {c} {\textbf{Benchmark, variable 3 as target}}\\
\ & $M=L=1$ & $M=L=2$ & $M=L=3$ \\
\hline 
Autoregression & $\{\}\rightarrow0.01$ & $\{\}\rightarrow0.01$ & $\{\}\rightarrow0.01$\\
Forward CMI & $\{0,4\}\rightarrow0.01$ & $\{1,2\}\rightarrow0$ & $\{2\}\rightarrow0$\\
Backward CMI & $\{2\}\rightarrow0$ & $\{1,2\}\rightarrow0$ & $\{2\}\rightarrow0$\\
FullCI\_linear & $\{2,4\}\rightarrow0.01$ & $\{2,4\}\rightarrow0.03$ & $\{2,4\}\rightarrow0.03$\\
FullCI\_nonLinear & $\{2\}\rightarrow0.01$ & $\{\}\rightarrow0.01$ & $\{1,2\}\rightarrow0.01$\\
PCMCI\_linear & $\{2,4\}\rightarrow0.01$ & $\{2,4\}\rightarrow0.03$ & $\{2,4\}\rightarrow0.03$\\
PCMCI\_nonLinear & $\{2\}\rightarrow0.01$ & $\{2,4\}\rightarrow0.03$ & $\{2,4\}\rightarrow0.03$\\
\cellcolor{lightgray!25} Forward TEFS (\textbf{ours}) & \cellcolor{lightgray!25} $\{4\}\rightarrow0.02$ & \cellcolor{lightgray!25} $\{\}\rightarrow0.01$ & \cellcolor{lightgray!25} $\{\}\rightarrow0.01$ \\
\cellcolor{lightgray!25} Backward TEFS (\textbf{ours}) & \cellcolor{lightgray!25} $\{4\}\rightarrow0.02$ & \cellcolor{lightgray!25} $\{\}\rightarrow0.01$ & \cellcolor{lightgray!25} $\{\}\rightarrow0.01$ \\
\hline
& \multicolumn{3} {c} {\textbf{Benchmark, variable 4 as target}}\\
\ & $M=L=1$ & $M=L=2$ & $M=L=3$ \\
\hline 
Autoregression & $\{\}\rightarrow0$ & $\{\}\rightarrow0$ & $\{\}\rightarrow0$ \\
Forward CMI & $\{2,3\}\rightarrow0$ & $\{2,3\}\rightarrow0.01$ & $\{2\}\rightarrow0.03$ \\
Backward CMI & $\{2,3\}\rightarrow0$ & $\{2,3\}\rightarrow0.01$ & $\{2\}\rightarrow0.03$ \\
FullCI\_linear & $\{2,3\}\rightarrow0$ & $\{2,3\}\rightarrow0.01$ & $\{2,3\}\rightarrow0.03$ \\
FullCI\_nonLinear & $\{3\}\rightarrow0$ & $\{2\}\rightarrow0.03$ & $\{2,3\}\rightarrow0.03$ \\
PCMCI\_linear & $\{2,3\}\rightarrow0$ & $\{2,3\}\rightarrow0.01$ & $\{2,3\}\rightarrow0.03$ \\
PCMCI\_nonLinear & $\{\}\rightarrow0$ & $\{2,3\}\rightarrow0.01$ & $\{2,3\}\rightarrow0.03$ \\
\cellcolor{lightgray!25} Forward TEFS (\textbf{ours}) & \cellcolor{lightgray!25} $\{2\}\rightarrow0$  & \cellcolor{lightgray!25} $\{2\}\rightarrow0.03$ & \cellcolor{lightgray!25} $\{2\}\rightarrow0.03$ \\
\cellcolor{lightgray!25} Backward TEFS (\textbf{ours}) & \cellcolor{lightgray!25} $\{2\}\rightarrow0$ & \cellcolor{lightgray!25} $\{2\}\rightarrow0.03$ & \cellcolor{lightgray!25} $\{2\}\rightarrow0.03$ \\
\hline
\end{tabular}
\end{table*}

Table \ref{tab:ClimRealFull} contains the subset of selected features and the $R^2$ test score with linear regression of the three climatological datasets with five features and the one with fifteen features, considering the three values of $L,M$ and the feature selection and causal discovery algorithms discussed above. From the results it is possible to see that the proposed algorithms almost always select the same set of features, showing a stable behaviour. Moreover, as already discussed, the two proposed methods tend to identify the subset of the most (causally) relevant features, preserving the majority of information, that is highlighted by the competitive results. PCMCI and FullCI algorithms have a less stable behaviour and tend to identify larget subsets of features, sometimes slightly improving the performance in terms of $R^2$ score, at the cost of considering more features. Finally, Table \ref{tab:BenchRealFull} reports the subset of selected features and the associated $R^2$ test score with linear regression of the benchmark dataset, considering, in turn, each variable as a target and the others as candidate causal features. From the results, the first variable (\emph{variable 0}) seems to be an autoregressive process. The second variable (\emph{variable 1}) seems to be a random variable with no dependencies on its previous value or the previous values of any other feature. For the third variable (\emph{variable 2}) all methods obtain similar performances, showing that the autoregressive component is negligible for this variable and \emph{variable 3 and 4} are the relevant ones for its prediction. Regarding the fourth variable (\emph{variable 3}), it is mostly a noise signal, with the proposed approaches that do not identify significant features in the majority of cases, while PCMCI and its variants identify slight importance of \emph{variable 2,4}. Finally, for the fifth variable (\emph{variable 4}), considering one timestep before, the other variables and their previous value do not provide information for its prediction, while a little information seems to be provided by \emph{variable 2} considering two and three timesteps before.

\newpage

\end{document}